\documentclass{article}

\usepackage{arxiv}

\usepackage[utf8]{inputenc} 
\usepackage[T1]{fontenc}    
\usepackage{hyperref}       
\usepackage{url}            
\usepackage{booktabs}       
\usepackage{amsfonts}       
\usepackage{nicefrac}       
\usepackage{microtype}      
\usepackage[pdftex]{graphicx}
\usepackage{amssymb,amsfonts,amsmath,amscd,amsthm}
\usepackage[printonlyused,withpage]{acronym}
\usepackage{natbib}
\usepackage{mathrsfs}
\usepackage{rotating}
\usepackage{cancel}

\newtheorem{theorem}{Theorem}[section]

\graphicspath{{figs/}, {figs/experiments/}}

\newcommand{\bbm}{\begin{bmatrix}}
\newcommand{\ebm}{\end{bmatrix}}
\newcommand{\mbf}{\mathbf}
\newcommand{\mbs}[1]{{\boldsymbol{#1}}}
\def\ep{\epsilon}

\newcommand{\beq}{\begin{equation}}
\newcommand{\eeq}{\end{equation}}
\newcommand{\bdis}{\begin{displaymath}}
\newcommand{\edis}{\end{displaymath}}
\newcommand{\beqn}[1]{\begin{subequations}\label{eq:#1}\begin{eqnarray}}
\newcommand{\eeqn}{\end{eqnarray}\end{subequations}}
\newcommand{\est}[1]{\hat{#1}}

\newcommand{\wdg}{\wedge}

\newcommand{\Wdg}{\curlywedge}

\newcommand{\Jbig}{\mbs{\mathcal{J}}}

\acrodef{BA}{Bundle Adjustment}
\acrodef{SLAM}{Simultaneous Localization and Mapping}
\acrodef{RANSAC}{Random Sample And Consensus}
\acrodef{QCQP}{Quadratically Constrained Quadratic Program}
\acrodef{SDP}{Semidefinite Program}
\acrodef{IRLS}{Iteratively Reweighted Least-Squares}
\acrodef{TLS}{Total Least-Squares}
\acrodef{SVD}{Singular-Value Decomposition}
\acrodef{GM}{Geman-McClure}
\acrodef{PSD}{positive-semidefinite}
\acrodef{GP}{Gaussian Process}
\acrodef{WNOA}{White Noise on Acceleration}
\acrodef{WNOJ}{White Noise on Jerk}
\acrodef{LICQ}{Linearly Independent Constraint Qualification}
\acrodef{SVR}{Singular Value Ratio}
\acrodef{LTV}{Linear Time-Varying}
\acrodef{LTI}{Linear Time-Invariant}
\acrodef{SDE}{Stochastic Differential Equation}
\acrodef{STEAM}{Simultaneous Trajectory Estimation And Mapping}

\hypersetup{%
    pdftitle={},
    pdfauthor={},
    pdfkeywords={},
    pdfsubject={},
    pdfstartview=FitH,%
    bookmarks=true,%
    bookmarksopen=true,%
    breaklinks=true,%
    colorlinks=true,%
    linkcolor=black,
    anchorcolor=black,%
    citecolor=black,
    filecolor=black,%
    menucolor=black,
    pagecolor=black,%
    urlcolor=black
}

\title{Revisiting Continuous-Time Trajectory Estimation \\ via Gaussian Processes and the Magnus Expansion}

\author{
 {\normalfont Timothy D. Barfoot} \\
 Robotics Institute \\
 University of Toronto \\
 \texttt{tim.barfoot@utoronto.ca} 
 \and
 {\normalfont Cedric Le Gentil} \\
 Robotics Institute \\
 University of Toronto \\
 \texttt{cedric.legentil@utoronto.ca} 
 \and 
 {\normalfont Sven Lilge} \\
 Robotics Institute \\
 University of Toronto \\
 \texttt{sven.lilge@utoronto.ca} 
}

\date{}

\begin{document}

\maketitle
\title{Revisiting Continuous-Time Trajectory Estimation via Gaussian Processes and the Magnus Expansion}

\begin{abstract}
Continuous-time state estimation has been shown to be an effective means of (i) handling asynchronous and high-rate measurements, (ii) introducing smoothness to the estimate, (iii) post hoc querying the estimate at times other than those of the measurements, and (iv) addressing certain observability issues related to scanning-while-moving sensors. A popular means of representing the trajectory in continuous time is via a \ac{GP} prior, with the prior's mean and covariance functions generated by a \ac{LTV} \ac{SDE} driven by white noise. When the state comprises elements of Lie groups, previous works have resorted to a patchwork of local \acp{GP} each with a \ac{LTI} \ac{SDE} kernel, which while effective in practice, lacks theoretical elegance.  Here we revisit the full \ac{LTV} \ac{GP} approach to continuous-time trajectory estimation, deriving a global \ac{GP} prior on Lie groups via the Magnus expansion, which offers a more elegant and general solution.  We provide a numerical comparison between the two approaches and discuss their relative merits.
\end{abstract}

\section{Introduction}

The ability to estimate trajectories in continuous time has become increasingly popular in the robotics and computer vision communities over the past decade. Figure~\ref{fig:ct} illustrates the general problem in which we are interested.  We have a trajectory $\mbf{x}(t)$ that we would like to estimate, but we only have measurements $\mbf{y}_k$ at discrete times $t_k$.  We would like our estimated trajectory to be smooth and physically plausible, which means placing a prior on the trajectory that encodes our knowledge of the underlying motion.  We would also like to be able to handle measurements that arrive at high rates and/or asynchronously.  Finally, we would like to be able to query the trajectory at any time of interest, $\tau$, not just at the measurement times.  This is in contrast to traditional discrete-time estimation approaches where the trajectory is only estimated at the measurement times.  

\begin{figure}[b]
    \centering
    \includegraphics[width=0.9\textwidth]{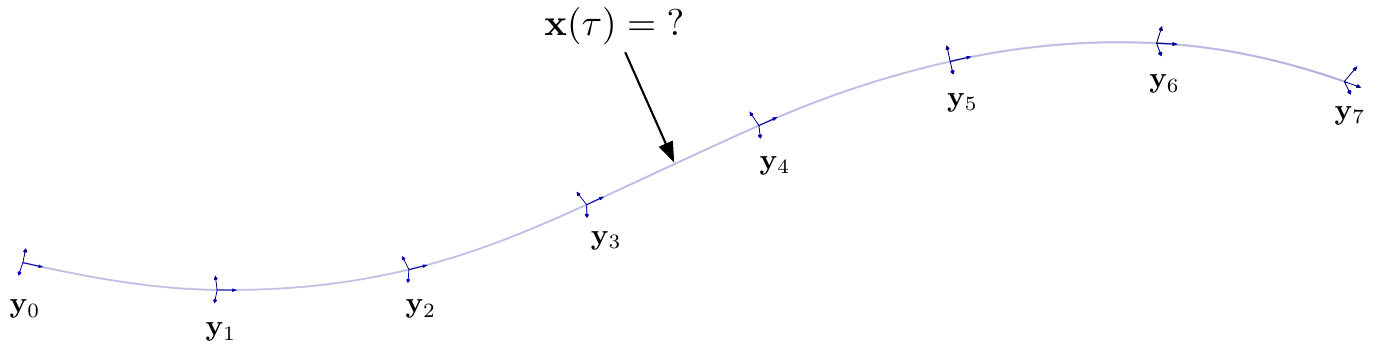}
    \caption{In continuous-time estimation, we consider that a sensor is moving smoothly through space over time.  At discrete times, the sensor takes measurements $\mbf{y}_k$ of the environment (e.g., images, lidar scans, etc.) that can be used to estimate the trajectory.  The trajectory itself is represented as a continuous-time function, $\mbf{x}(t)$, often using a \ac{GP} prior, regularizing the trajectory to be smooth and physically plausible.}
    \label{fig:ct}
\end{figure}

Continuous-time estimation addresses all of these requirements in a natural way:
\begin{itemize}
\item {\em We can easily handle high-rate and asynchronous measurements.}  This is accomplished by assigning each measurement to the trajectory at its time of acquisition, $t_k$.\footnote{Nominally, we end up with a discrete state at each measurement time; however, there are ways to reduce the number of states if necessary \citep{tong_ijrr13b,anderson_iros15}.}

\item {\em We can query the trajectory at any time of interest, $\tau$, not just at the measurement times, $t_k$.}  This is useful for applications where we might use one sensor to estimate the trajectory but another to build a map, for example.  It is also useful in control where we might want to query the trajectory at a high rate to compute control inputs. 

\item {\em The continuous-time representation can impose smoothness on the estimated trajectory.}  Often the object to which sensors are attached is moving smoothly through the world because it is governed by physics, and we can use this knowledge to regularize the solution through the choice of motion prior.

\item {\em The continuous-time representation can be used to overcome observability issues.}  For sweeping-while-moving sensors such as rolling-shutter cameras or spinning lidars, the trajectory is often under-constrained by the measurements alone since each pixel or point is acquired from a unique pose.  Our motion prior serves to regularize the solution and can be physically motivated, such as a constant-velocity prior or a constant-acceleration prior.
\end{itemize}

There are essentially two main ways to represent the continuous-time trajectory: (i) using a parametric representation such as splines, or (ii) using a non-parametric representation such as a \ac{GP}.  Both approaches have their merits, and in this work we focus on the \ac{GP} approach.  For a recent detailed literature review of both approaches, we refer the reader to \citet{talbot_tro25}.  We will provide a summary of the \ac{GP} literature here.

The seminal reference for \acp{GP} is \citet{rasmussen06}; this contains some discussion of estimation for dynamic systems, but they do not extend to Lie groups.  In the context of filtering and smoothing, \citet{sarkka06} showed the equivalence between discrete-time estimation and continuous-time estimation at the measurement times (i.e., chain-like graph structures), which is quite related to the \ac{GP} approach we discuss here.  

Within robotics, \citet{tong_crv12, tong_ijrr13b} made the initial connection between \ac{GP} regression and continuous-time estimation for general batch estimation problems (i.e., arbitrarily connected graph structures such as \ac{SLAM}).  This was followed up by \citet{barfoot_rss14,anderson_ar15}, who showed how to construct a kernel function from a physically motivated motion prior that resulted in a block-tridiagonal inverse kernel matrix, which is key to efficient computation; they also made the initial connection between the \ac{GP} approach and sparse factor graphs and coined the term \ac{STEAM}.  \citet{anderson_iros15, anderson_phd16} showed how to apply the \ac{GP} continuous-time approach (specifically a \ac{WNOA} motion model) on Lie groups using the `local GP' approach, which we use as the baseline in this article; all follow-on works that work with \acp{GP} on Lie group have used the local approach.

\citet{yanIncrementalSparseGP2015, yanIncrementalSparseGP2017} showed that once the \ac{GP} approach was formulated as a factor graph it could be incrementalized using the Bayes-tree ideas of \citet{kaess08, kaessISAM2IncrementalSmoothing2012}.  \citet{dongSparseGaussianProcesses2018} generalized the \ac{GP} formulation to different Lie groups and provided examples.  \citet{mukadamContinuoustimeGaussianProcess2018} extended the approach beyond estimation to include motion-planning problems.

\citet{tang_ral19}, \citet{nguyen_tro25} showed how to construct a \ac{WNOJ} prior for Lie groups that worked well for vehicles experiencing acceleration.  \citet{wong_ral20} showed how to construct a Singer prior for Lie groups that was trained on data to better fit a real-world motion profile. \citet{legentilGaussianProcessPreintegration2020, legentilContinuousLatentState2023} showed how to use \acp{GP} to preintegrate inertial measurements in continuous time; this work is not directly connected to the \ac{GP} approach discussed in this article but is an important related topic.  

\citet{lilge_ijrr22,lilge_tro25} showed that the \ac{GP} approach could be used to also estimate the shape of continuum robots. \citet{duembgen_ral23,barfoot_ijrr25} showed how to incorporate \ac{GP} motion priors within a certifiably optimal estimation framework, including on Lie groups. \citet{johnson_arxiv24} provided a detailed experimental comparison of parametric and nonparametric methods for continuous-time estimation and showed how to incorporate motion priors into parametric methods to make them essentially equivalent to nonparametric methods. 

\citet{barfoot_ser24} provided a new way to carry out covariance interpolation during querying that was simpler than \citet{anderson_phd16}.  \citet{lilge_rob25} extended the Lie-group \ac{GP} motion priors to include non-zero mean functions, allowing more specific prior motion to be captured.  \citet{burnett_tro25} showed state-of-the-art performance on radar- and lidar-inertial estimation using the \ac{GP} approach.

All of these methods resort to the local \ac{GP} approach of \citet{anderson_iros15} when working with Lie groups.  In this article, we revisit this choice and show how to construct a global \ac{GP} prior on Lie groups by directly solving the \ac{LTV} \ac{SDE} using the Magnus expansion \citep{magnus54,blanes09}.  This leads to a more elegant and general solution that avoids the patchwork of local \acp{GP}.  

The rest of this article is organized as follows.  In Section~\ref{sec:setup}, we set up the problem and provide the necessary background.  In Section~\ref{sec:magnus}, we introduce the Magnus expansion and show how to use it to solve the \ac{LTV} \ac{SDE} that arises from linearizing the motion model on Lie groups.  In Section~\ref{sec:gpprior}, we show how to construct the global \ac{GP} prior using the results from Section~\ref{sec:magnus}.  In Section~\ref{sec:fg_impl}, we discuss implementation details.  In Section~\ref{sec:numerical_example}, we provide some experimental results.  Finally, in Section~\ref{sec:conclusion}, we conclude and discuss future work.


\section{Setup and Problem Statement}
\label{sec:setup}

\begin{figure}[t]
    \centering
    \includegraphics[width=\textwidth]{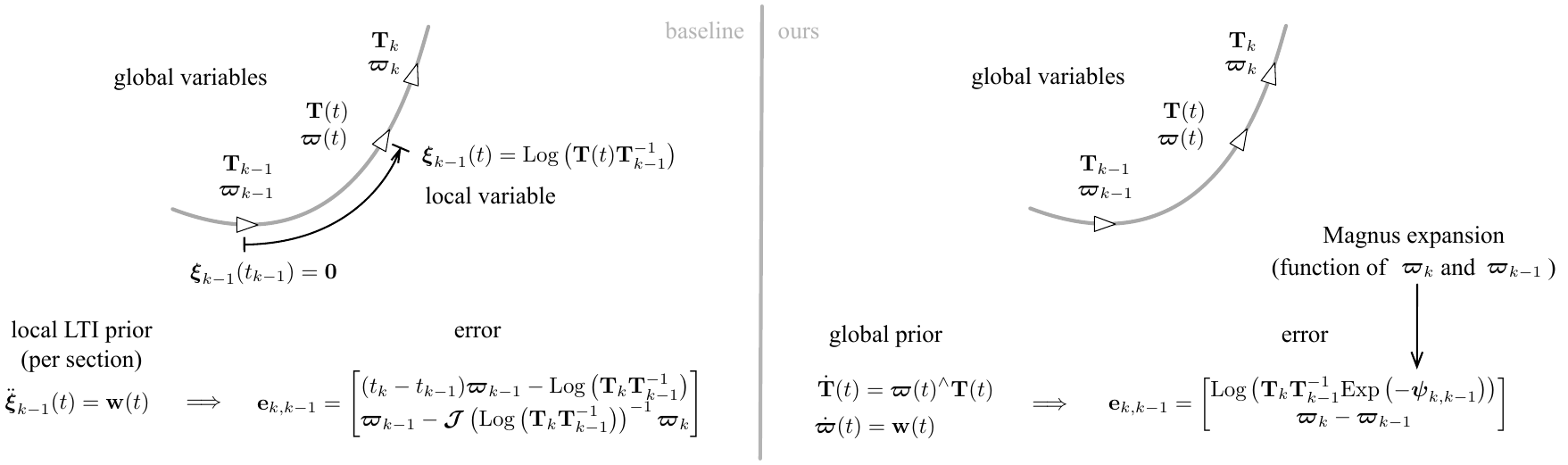}
    \caption{High-level comparison of two different ways to define a Gaussian process prior for continuous-time trajectory estimation: (left) the baseline approach using local \ac{LTI} \acp{SDE} stitched together, and (right) the proposed approach using a global \ac{SDE}.}
    \label{fig:gpcompare}
\end{figure}

A \ac{GP} prior can be generated by a \ac{LTV} \ac{SDE} as in
\begin{equation}
    \dot{\mbf{x}}(t) = \mbf{A}(t) \mbf{x}(t) + \mbf{L}(t) \mbf{w}(t), \quad \mbf{w}(t) \sim \mathcal{GP}(\mbf{0}, \mbf{Q}_c \delta(t-t')),
    \label{eq:ltv_sde}
\end{equation}
where $\mbf{x}(t)$ is the state, $\mbf{A}(t)$ is the (possibly time-varying) system matrix, $\mbf{L}(t)$ is the (possibly time-varying) noise input matrix, and $\mbf{w}(t)$ is a white process noise with power spectral density $\mbf{Q}_c$.  However, when the state $\mbf{x}(t)$ lies on a Lie group (e.g., $SE(3)$), then \eqref{eq:ltv_sde} is not directly applicable.  

Instead, our desired motion model takes the form (e.g., white-noise on acceleration)
\begin{equation}
    \dot{\mbf{T}}(t) = \mbs{\varpi}(t)^\wedge \mbf{T}(t), \quad \dot{\mbs{\varpi}}(t) = \mbf{w}(t), \quad \mbf{w}(t) \sim \mathcal{GP}(\mbf{0}, \mbf{Q}_c \delta(t-t')),
    \label{eq:lie_sde}
\end{equation}
where $\mbf{T}(t) \in SE(3)$ is the pose, $\mbs{\varpi}(t) \in \mathbb{R}^6$ is the body-centric velocity,  and $(\cdot)^\wedge$ is the operator that maps a vector in $\mathbb{R}^6$ to an element of the Lie algebra $\mathfrak{se}(3)$\footnote{We use the notation of \citep{barfoot_ser24} when discussing Lie groups plus the shorthands $\mbox{Exp}(\cdot) = \exp((\cdot)^\wdg)$ and $\mbox{Log}(\cdot) = \mbox{ln}(\cdot)^\vee$ \citep{sola18}.}.  Solving \eqref{eq:lie_sde} directly is difficult due to the nature of the Lie group and involves complex propagation of uncertainty on manifolds via the Fokker-Planck equation.  To keep things tangible, we will use $SE(3)$ as our running example throughout this work, but the ideas presented here are applicable to other matrix Lie groups as well.

If we linearize the motion model about a nominal trajectory $\left\{ \mbf{T}_{\rm op}(t), \mbs{\varpi}_{\rm op}(t) \right\}$, this leads to an \ac{LTV} \ac{SDE}, which is still difficult to solve directly.  A common workaround in the literature is to instead use a stitched sequence of \ac{LTI} \acp{SDE} \citep{anderson_iros15,barfoot_ser24}; however, while this can work well in practice, it is inelegant and can be difficult to carry out querying at arbitrary times.  Here we propose to use the Magnus expansion \citep{magnus54,blanes09} to directly solve the \ac{LTV} \ac{SDE} that arises from linearizing \eqref{eq:lie_sde} about a nominal trajectory, leading to a single global \ac{GP} prior on the Lie group.  The details of this are nontrivial and comprise the main contribution of this work. 

Ultimately, what we hope to achieve with our \ac{GP} prior is to create binary factors between pairs of states along the trajectory that can be used in a batch estimation framework.  These binary factors take the form
\begin{equation}
    || \mbf{e}_{k,k-1} ||_{\mbf{Q}_{k,k-1}}^2,
\end{equation}
where $\mbf{e}_{k,k-1}$ is the error between the state at time $t_k$ and the state at time $t_{k-1}$, and $\mbf{Q}_{k,k-1}$ is the corresponding covariance.  Figure~\ref{fig:gpcompare} contrasts our proposed approach with the baseline approach using local \ac{LTI} \acp{SDE}; the difference comes down to the formulation of the motion-prior factors.  When the motion prior is combined with other factors based on sensor measurements, the state estimation problem amounts to an optimization problem, which we typically minimize iteratively using a variant of a Gauss-Newton solver. We thus have three tasks in our immediate future: (i) discretize~\eqref{eq:lie_sde} temporally, (ii) formulate the error $\mbf{e}_{k,k-1}$ and covariance $\mbf{Q}_{k,k-1}$, and (iii) linearize the error about the current trajectory estimate to prepare for Gauss-Newton optimization. 

However, it is not obvious in what order we should carry out these three tasks.  Figure~\ref{fig:commdiag} shows a commutative diagram illustrating the various pathways from the desired \ac{GP} prior motion model in the top left to the linearized error in the bottom right.  Should we discretize temporally first, then form the error, then linearize?  Or should we linearize first, then discretize temporally, then form the error? The answers to these questions will have important implications for the implementation of the \ac{GP} prior in practice.  It turns out that all pathways lead to the same result if we are careful to introduce consistent approximations.


\section{Magnus Expansion Background}
\label{sec:magnus}

Central to our development is the use of the Magnus expansion \citep{magnus54,blanes09}, which allows us to solve an \ac{LTV} differential equation of the form 
\begin{equation}
    \dot{\mbf{X}}(t) = \mbf{A}(t) \mbf{X}(t)
\end{equation}
in a principled manner, particularly when $\mbf{A}(t)$ is noncommutative.  The Magnus expansion states that the solution can be written as
\begin{equation}
    \mbf{X}_k = \exp\left( \mbs{\Omega}_{k,k-1} \right) \mbf{X}_{k-1},
\end{equation}
where $\mbf{X}_k = \mbf{X}(t_k)$ and the `Magnus matrix', $\mbs{\Omega}_{k,k-1} = \sum_{i=1}^\infty \mbs{\Omega}_{i,k,k-1}$, is given by an infinite series of integrals and commutators of $\mbf{A}(t)$ evaluated at different times.  Appendix~\ref{app:magnus} provides a brief derivation of the Magnus expansion as we require it.

While the Magnus expansion is very general and its full power could be brought to bear on our problem, in the interest of simplicity we will restrict ourselves to the case where $\mbf{A}(t)$ varies linearly with time from $t_{k-1}$ to $t_k$, which is appropriate for small time intervals and our running example of a \ac{WNOA} motion model on $SE(3)$\footnote{If we were keeping more derivatives in our state (e.g., \ac{WNOJ}) then we could use more sophisticated approximations of $\mbf{A}(t)$ (e.g., cubic).}.  In this case, we can write
\begin{equation}\label{eq:At_linear}
    \mbf{A}(t) = \mbf{A}_{k-1} + \frac{t - t_{k-1}}{\Delta t_k} \left(\mbf{A}_k - \mbf{A}_{k-1} \right),
\end{equation}
where $\Delta t_k = t_k - t_{k-1}$.  With this assumption, the integrals in the Magnus expansion can be evaluated analytically, leading to closed-form expressions for $\mbs{\Omega}_{k,k-1}$ up to any desired order.  The first four terms in this case are
\begin{subequations}\label{eq:magnus_terms}
\begin{align}
    \mbs{\Omega}_{1,k,k-1} &= \frac{\Delta t_k}{2}  \left( \mbf{A}_{k-1} + \mbf{A}_k \right), \\
    \mbs{\Omega}_{2,k,k-1} &= \frac{\Delta t_k^2}{12}  [\mbf{A}_k, \mbf{A}_{k-1}], \\
    \mbs{\Omega}_{3,k,k-1} &= \frac{\Delta t_k^3 }{240} [\mbf{A}_k - \mbf{A}_{k-1}, [\mbf{A}_k, \mbf{A}_{k-1}]], \\
    \mbs{\Omega}_{4,k,k-1} &= -\frac{\Delta t_k^4}{5040}  [\mbf{A}_k - \mbf{A}_{k-1}, [\mbf{A}_k - \mbf{A}_{k-1}, [\mbf{A}_k, \mbf{A}_{k-1}]]] - \frac{\Delta t_k^4}{720}  [\mbf{A}_k, [\mbf{A}_{k-1}, [\mbf{A}_k, \mbf{A}_{k-1}]]],
\end{align}
\end{subequations}
where $[\mbf{A}, \mbf{B}] = \mbf{A}\mbf{B} - \mbf{B}\mbf{A}$ is the matrix commutator (also the Lie bracket).  By including higher-order terms in the expansion, we can achieve greater accuracy in the solution. 
This will allow us to derive the discrete-time state transition matrix and process noise covariance needed for our \ac{GP} prior.

In the specific case that $\mbf{A}(t) = \mbs{\varpi}(t)^\wedge$ as in \eqref{eq:lie_sde}, the Magnus expansion will yield a solution that respects the Lie group structure, allowing us to propagate the state and its uncertainty in a manner consistent with the underlying geometry of the problem.  In detail, we have that 
\begin{equation}
\mbf{T}_k = \mbox{Exp}\left( \mbs{\psi}_{k,k-1} \right) \mbf{T}_{k-1},
\end{equation} 
where \citep{huber20,barfoot_ser24}
\begin{equation}
    \mbs{\psi}_{k,k-1} = \frac{\Delta t_k }{2} \left( \mbs{\varpi}_{k} + \mbs{\varpi}_{k-1} \right) + \frac{\Delta t_k^2}{12}  \mbs{\varpi}_k^\Wdg \mbs{\varpi}_{k-1} + \frac{\Delta t_k^3}{240} \left(\mbs{\varpi}_k - \mbs{\varpi}_{k-1}\right)^\Wdg \mbs{\varpi}_{k}^\Wdg \mbs{\varpi}_{k-1} + \cdots, 
\end{equation}
which we refer to as a `Magnus vector' and where $(\cdot)^\Wdg = \mbox{ad}\left((\cdot)^\wdg\right)$ is the adjoint operator that maps a vector in $\mathbb{R}^6$ to the adjoint representation of the Lie algebra $\mathfrak{se}(3)$.
We next use these results to build our \ac{GP} prior on Lie groups.  The strategy will be to delay choosing the number of terms in the Magnus expansion until the very end, allowing us to trade off accuracy and computational cost.


\section{Building the GP Prior on Lie Groups}
\label{sec:gpprior}

\begin{figure}[t]
    \centering
    \includegraphics[width=\textwidth]{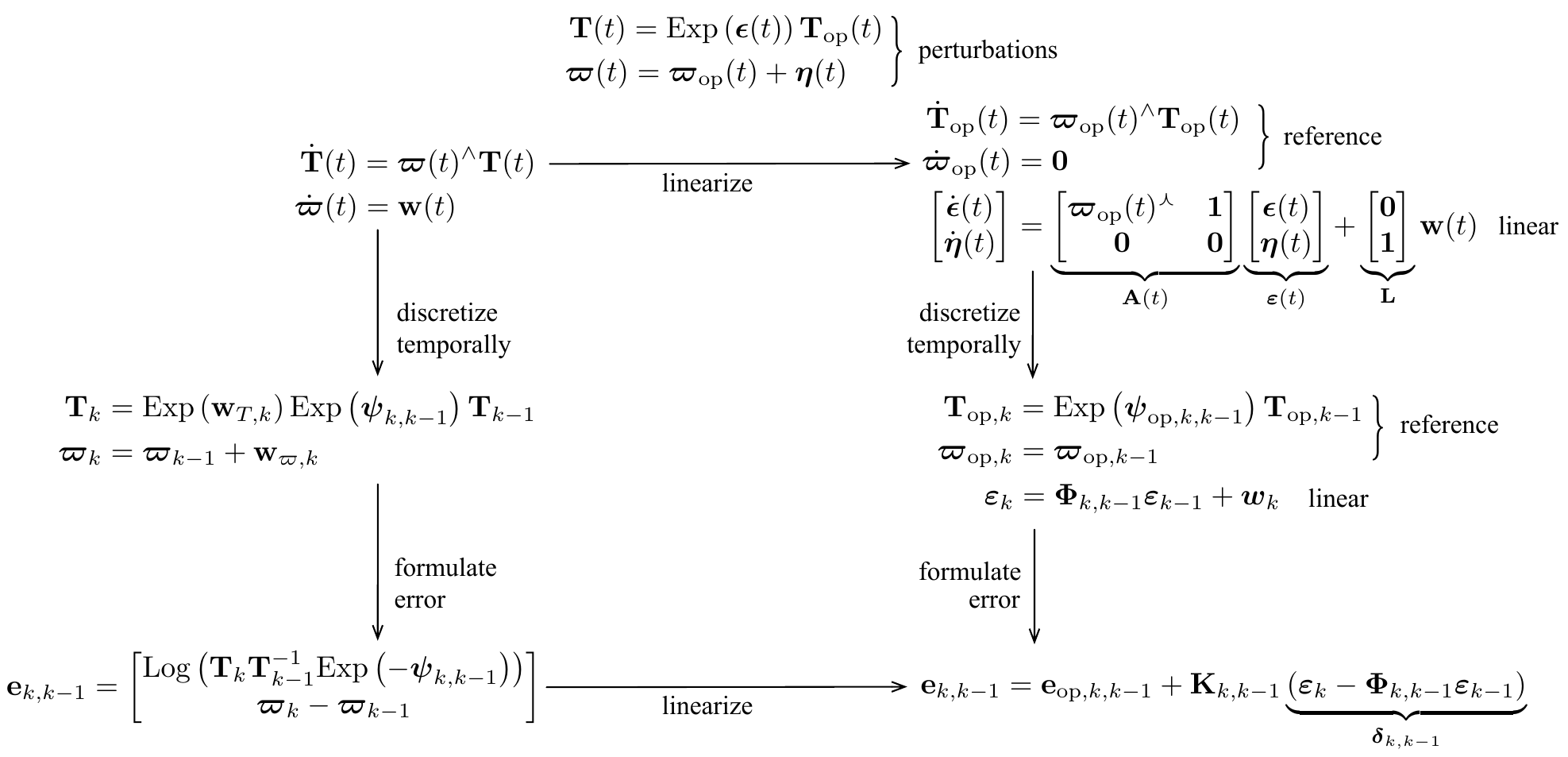}
    \caption{Commutative diagram showing the various pathways from the desired continuous-time Gaussian process prior motion model in the top left to the discrete-time linearized error in the bottom right.  The key to making this diagram self-consistent is the use of the Magnus expansion.  Symbols explained in the text.}
    \label{fig:commdiag}
\end{figure}

We will use Figure~\ref{fig:commdiag} as a guide to building our \ac{GP} prior on Lie groups.  We begin at the top left with the motion model in \eqref{eq:lie_sde} and initially proceed counterclockwise around the diagram.  
We have seen in the previous section that using the Magnus expansion allows us to discretize the \ac{LTV} \ac{SDE} from~\eqref{eq:lie_sde} directly, leading to a discrete-time motion model.  Accounting for the process noise, we have
\beqn{}
    \mbf{T}_k & = & \mbox{Exp}\left( \mbf{w}_{T,k}\right) \mbox{Exp}\left( \mbs{\psi}_{k,k-1} \right) \mbf{T}_{k-1}, \\
    \mbs{\varpi}_k & = & \mbs{\varpi}_{k-1} + \mbf{w}_{\varpi,k}, 
\eeqn
where the discrete-time process noise obeys
\begin{equation}
    \mbf{w}_k = \begin{bmatrix} \mbf{w}_{T,k} \\ \mbf{w}_{\varpi,k} \end{bmatrix} \sim \mathcal{N}\left( \mbf{0}, \mbf{Q}_{k,k-1} \right),
\end{equation}
with $\mbf{Q}_{k,k-1}$ the discrete-time process noise covariance still to be determined; it will be a function of $\mbf{Q}_c$. Next, we form the error between the states at times $t_{k-1}$ and $t_k$ as
\begin{equation}\label{eq:err}
    \mbf{e}_{k,k-1} = \bbm \mbf{e}_{T,k,k-1} \\ \mbf{e}_{\varpi,k,k-1} \ebm = \begin{bmatrix} \mbox{Log}\left( \mbf{T}_k \mbf{T}_{k-1}^{-1} \text{Exp}\left( -\mbs{\psi}_{k,k-1} \right) \right) \\ \mbs{\varpi}_k - \mbs{\varpi}_{k-1} \end{bmatrix}.
\end{equation}
Our binary factor is then given by $|| \mbf{e}_{k,k-1} ||_{\mbf{Q}_{k,k-1}}^2$, as desired.

Next, we need to linearize this error about the current trajectory estimate $\left\{ \mbf{T}_{\rm op}(t), \mbs{\varpi}_{\rm op}(t) \right\}$.  We introduce perturbations as
\beqn{perts}
    \mbf{T}_k & = & \mbox{Exp}\left( \mbs{\ep}_k \right) \mbf{T}_{{\rm op},k}, \\
    \mbs{\varpi}_k & = &\mbs{\varpi}_{{\rm op},k} + \mbs{\eta}_k,  
\eeqn
The tricky part of this is that the `Magnus vector', $\mbs{\psi}_{k,k-1}$, depends on the body-centric velocities at both $t_{k-1}$ and $t_k$, which themselves are being perturbed.  For now we will linearize $\mbs{\psi}_{k,k-1}$ as
\begin{equation}\label{eq:man_lin}
    \mbs{\psi}_{k,k-1} \approx \mbs{\psi}_{{\rm op},k,k-1} + \mbf{M}_k \mbs{\eta}_k + \mbf{M}_{k-1} \mbs{\eta}_{k-1},
\end{equation}
where $\mbf{M}_k$ and $\mbf{M}_{k-1}$ are Jacobians to be determined based on the number of terms retained in the Magnus expansion.  Substituting \eqref{eq:perts} and \eqref{eq:man_lin} into \eqref{eq:err} and linearizing leads to the desired linearized error:
\begin{multline}
    \mbf{e}_{k,k-1} \approx \mbf{e}_{{\rm op},k,k-1} + \bbm \Jbig(\mbf{e}_{{\rm op},T,k,k-1})^{-1} & \mbf{0} \\ \mbf{0} & \mbf{1} \ebm \left( \bbm \mbf{1} & -\Jbig\left(\mbs{\psi}_{{\rm op},k,k-1}\right) \mbf{M}_k \\ \mbf{0} & \mbf{1} \ebm \bbm \mbs{\ep}_k \\ \mbs{\eta}_k \ebm \right. \\
    \left. - \bbm \exp\left(\mbs{\psi}_{{\rm op},k,k-1}^\Wdg \right) & \Jbig\left(\mbs{\psi}_{{\rm op},k,k-1}\right) \mbf{M}_{k-1} \\ \mbf{0} & \mbf{1} \ebm \bbm \mbs{\ep}_{k-1} \\ \mbs{\eta}_{k-1} \ebm \right),
\end{multline}
where $\Jbig(\cdot)$ is the left Jacobian of $SE(3)$ \citep{barfoot_ser24}. Next, we factor out one of the coefficient matrices to be able to write 
\begin{multline}\label{eq:lin_err_ccw}
    \mbf{e}_{k,k-1} \approx \mbf{e}_{{\rm op},k,k-1} + \underbrace{\bbm \Jbig(\mbf{e}_{{\rm op},T,k,k-1})^{-1} & \mbf{0} \\ \mbf{0} & \mbf{1} \ebm \bbm \mbf{1} & -\Jbig\left(\mbs{\psi}_{{\rm op},k,k-1}\right) \mbf{M}_k \\ \mbf{0} & \mbf{1} \ebm}_{\mbf{K}_{k}}   \\
    \times \; \Biggl(  \underbrace{\bbm \mbs{\ep}_k \\ \mbs{\eta}_k \ebm}_{\mbs{\varepsilon}_k}  - \underbrace{\bbm \exp\left(\mbs{\psi}_{{\rm op},k,k-1}^\Wdg \right) &  \Jbig\left(\mbs{\psi}_{{\rm op},k,k-1}\right) \left( \mbf{M}_k + \mbf{M}_{k-1} \right) \\ \mbf{0} & \mbf{1} \ebm}_{\mbs{\Phi}_{k,k-1}} \underbrace{\bbm \mbs{\ep}_{k-1} \\ \mbs{\eta}_{k-1} \ebm}_{\mbs{\varepsilon}_{k-1}} \Biggr)
\end{multline}
or simply 
\begin{equation}
    \mbf{e}_{k,k-1} \approx \mbf{e}_{{\rm op},k,k-1} + \mbf{K}_{k} \, \mbs{\delta}_{k,k-1}, \quad \mbs{\delta}_{k,k-1} = \mbs{\varepsilon}_k - \mbs{\Phi}_{k,k-1} \mbs{\varepsilon}_{k-1}.
\end{equation}
This is our desired linearized error expression, completing the counterclockwise path around the commutative diagram in Figure~\ref{fig:commdiag}.  Importantly, we see that the perturbation variables form another kind of error, $\mbs{\delta}_{k,k-1}$, and the matrix $\mbf{K}_{k}$ converts this error into the original error space.  

We next consider the clockwise path around the diagram, starting again at the top left with \eqref{eq:lie_sde}.  We first linearize the motion model about the nominal trajectory, leading to a separation between the reference trajectory,
\begin{equation}
    \dot{\mbf{T}}_{\rm op}(t) = \mbs{\varpi}_{\rm op}(t)^\wedge \mbf{T}_{\rm op}(t), \quad \dot{\mbs{\varpi}}_{\rm op}(t) = \mbf{0},
\end{equation}
and the linear model in terms of the perturbations,
\begin{equation}\label{eq:pert_sde}
    \bbm \dot{\mbs{\ep}}(t) \\ \dot{\mbs{\eta}}(t) \ebm = \underbrace{\bbm \mbs{\varpi}_{\rm op}(t)^\Wdg & \mbs{1} \\ \mbf{0} & \mbf{0} \ebm}_{\mbf{A}(t)} \underbrace{\bbm \mbs{\ep}(t) \\ \mbs{\eta}(t) \ebm}_{\mbs{\varepsilon}(t)}  + \underbrace{\bbm \mbf{0} \\ \mbf{1} \ebm}_{\mbf{L}} \mbf{w}(t),
\end{equation}
which is exactly in the form of~\eqref{eq:ltv_sde}.  Now, if we discretize this \ac{LTV} \ac{SDE} temporally using the Magnus expansion as before, we obtain for the reference
\beqn{}
    \mbf{T}_{{\rm op},k} & = & \mbox{Exp}\left( \mbs{\psi}_{{\rm op},k,k-1} \right) \mbf{T}_{{\rm op},k-1}, \\
    \mbs{\varpi}_{{\rm op},k} & = & \mbs{\varpi}_{{\rm op},k-1}.
\eeqn
For the \ac{LTV} \ac{SDE} in the perturbations, we must employ the Magnus expansion again but now we must work with $\mbf{A}(t)$ defined above.  We assume that we know $\mbf{A}(t)$ at the discrete times $t_{k-1}$ and $t_k$ based on the reference trajectory, and we again use the linear-in-time approximation from \eqref{eq:At_linear}.  {\em What is remarkable is that if we use the same number of terms in the Magnus expansion as we did in $\mbs{\psi}_{k,k-1}$, we obtain exactly the same discrete-time state transition matrix, $\mbs{\Phi}_{k,k-1}$, as we did in the counterclockwise path!}  This is not a coincidence; it is a direct consequence of the properties of the Magnus expansion.  Appendix~\ref{app:commdiag_proof} provides a proof of the equivalence of the discrete-time state transition matrices obtained via the clockwise and counterclockwise paths.  Thus, we have    
\begin{equation}\label{eq:stm}
    \mbs{\Phi}_{k,k-1} = \bbm \exp\left(\mbs{\psi}_{{\rm op},k,k-1}^\Wdg \right) &  \Jbig\left(\mbs{\psi}_{{\rm op},k,k-1}\right) \left( \mbf{M}_k + \mbf{M}_{k-1} \right) \\ \mbf{0} & \mbf{1} \ebm,
\end{equation}
where $\mbf{M}_k$ and $\mbf{M}_{k-1}$ are the same Jacobians as before and then 
\begin{equation}
    \mbs{\varepsilon}_k = \mbs{\Phi}_{k,k-1} \mbs{\varepsilon}_{k-1} + \mbs{w}_k,
\end{equation}
where $\mbs{w}_k \sim \mathcal{N}\left( \mbf{0}, \mbs{Q}_{k,k-1} \right)$ is a discrete-time process noise (different from $\mbf{w}_k$). Now, we can form the reference and linearized errors using the discrete-time models to be
\beqn{}
    \mbf{e}_{{\rm op},k,k-1} & = & \bbm \mbox{Log}\left( \mbf{T}_{{\rm op},k} \mbf{T}_{{\rm op},k-1}^{-1} \text{Exp}\left( -\mbs{\psi}_{{\rm op},k,k-1} \right) \right) \\ \mbs{\varpi}_{{\rm op},k} - \mbs{\varpi}_{{\rm op},k-1} \ebm, \\
    \mbs{\delta}_{k,k-1} & = & \mbs{\varepsilon}_k - \mbs{\Phi}_{k,k-1} \mbs{\varepsilon}_{k-1}. 
\eeqn
While this completes the clockwise path around the commutative diagram in Figure~\ref{fig:commdiag}, we still need to determine the discrete-time process noise covariance $\mbf{Q}_{k,k-1}$ to complete the \ac{GP} prior.

It turns out that both the clockwise and counterclockwise paths provide us with some advantages.  The counterclockwise path is important because it allows us to determine the matrix $\mbf{K}_{k}$ that maps the error in the perturbation space to the original error space.  The clockwise path is important because it allows us to determine the covariance associated with the process noise $\mbs{w}_k$, again using the Magnus expansion.  Our desired $\mbf{Q}_{k,k-1}$ is then simply given by
\begin{equation} \label{eq:Qk_conversion}
    \underbrace{E[\mbf{w}_k \mbf{w}_k^T]}_{\mbf{Q}_{k,k-1}} = \mbf{K}_{k} \, \underbrace{E[\mbs{w}_k \mbs{w}_k^T]}_{\mbs{Q}_{k,k-1}} \, \mbf{K}_{k}^T.
\end{equation}
We will focus on computing $\mbs{Q}_{k,k-1}$ in the next section.


\section{Calculating the Discrete-Time Process Noise Covariance}
\label{sec:proc_noise_cov}

As mentioned at the end of the last section, we will focus on calculating $\mbs{Q}_{k,k-1} = E[\mbs{w}_k \mbs{w}_k^T]$ using the Magnus expansion applied to the \ac{LTV} \ac{SDE} in \eqref{eq:pert_sde}, and then use \eqref{eq:Qk_conversion} to obtain $\mbf{Q}_{k,k-1}$.  The process noise covariance for an \ac{LTV} \ac{SDE} of the form in \eqref{eq:ltv_sde} is given by \citep{barfoot_ser24}
\begin{equation}\label{eq:proc_noise_cov}
    \mbs{Q}_{k,k-1} = \int_{t_{k-1}}^{t_k} \mbs{\Phi}(t_k, \tau) \mbf{L} \mbf{Q}_c \mbf{L}^T \mbs{\Phi}(t_k, \tau)^T d\tau,
\end{equation}
where $\mbs{\Phi}(t_k, \tau)$ is the state transition matrix from time $\tau$ to time $t_k$.  

There are a few choices for how to compute this integral.  
One straightforward approach is to break the integral into $N$ smaller pieces as follows:
\begin{equation}
    \mbs{Q}_{k,k-1} = \sum_{n=1}^N \mbs{\Phi}(t_k, s_n) \, \underbrace{\int_{s_{n-1}}^{s_n} \mbs{\Phi}(s_n, \tau) \mbf{L} \mbf{Q}_c \mbf{L}^T \mbs{\Phi}(s_n, \tau)^T d\tau }_{\mbs{Q}_n}\,\, \mbs{\Phi}(t_k, s_n)^T,
\end{equation}
where $s_0 = t_{k-1}$, $s_n = t_k$, and $s_n - s_{n-1} = h$ with $h = (t_k - t_{k-1})/N$ a constant.  The transition function $\mbs{\Phi}(t_k,s_n)$ can be evaluated using the Magnus expansion, for example.  If we make $N$ large enough, we can approximate the integral $\mbs{Q}_n$ in each subinterval.  If we assume that $\mbs{\varpi}(t) = \mbs{\varpi}_n$ is constant over each subinterval, then $\mbf{A}(t)$ is also constant and then the transition matrix is simply 
\begin{equation}
    \mbs{\Phi}(s_n, \tau) = \bbm \exp\left( (s_n-\tau) \mbs{\varpi}_n^\Wdg \right) & (s_n-\tau) \Jbig((s_n-\tau) \mbs{\varpi}_n) \\ \mbf{0} & \mbf{1} \ebm.
\end{equation}
The integrand becomes 
\begin{multline}
    \mbs{\Phi}(s_n, \tau) \mbf{L} \mbf{Q}_c \mbf{L}^T \mbs{\Phi}(s_n, \tau)^T = \bbm (s_n-\tau)^2 \Jbig((s_n-\tau) \mbs{\varpi}_n) \, \mbf{Q}_c \, \Jbig((s_n-\tau) \mbs{\varpi}_n)^T & (s_n-\tau) \Jbig((s_n-\tau) \mbs{\varpi}_n) \, \mbf{Q}_c \\ (s_n-\tau) \mbf{Q}_c \, \Jbig((s_n-\tau) \mbs{\varpi}_n)^T & \mbf{Q}_c \ebm \\
    = \bbm (s_n-\tau)^2 \mbf{Q}_c + \frac{1}{2} (s_n-\tau)^3 \left( \mbs{\varpi}_n^\Wdg \mbf{Q}_c + \mbf{Q}_c \mbs{\varpi}_n^{\Wdg^T} \right) + \cdots  & (s_n-\tau) \mbf{Q}_c + \frac{1}{2} (s_n-\tau)^2 \mbs{\varpi}_n^\Wdg \mbf{Q}_c + \cdots  \\ (s_n-\tau) \mbf{Q}_c + \frac{1}{2} (s_n-\tau)^2 \mbf{Q}_c \mbs{\varpi}_n^{\Wdg^T} + \cdots & \mbf{Q}_c \ebm,
\end{multline}
where we have expanded the left Jacobian in a Taylor series.  Integrating term-by-term from $\tau = s_{n-1}$ to $\tau = s_n$ gives
\begin{equation}
\mbs{Q}_n = \bbm \frac{h^3}{3} \mbf{Q}_c + \frac{h^4}{8} \left( \mbs{\varpi}_n^\Wdg \mbf{Q}_c + \mbf{Q}_c \mbs{\varpi}_n^{\Wdg^T} \right) + \cdots & \frac{h^2}{2} \mbf{Q}_c + \frac{h^3}{6} \mbs{\varpi}_n^\Wdg \mbf{Q}_c + \cdots  \\ \frac{h^2}{2} \mbf{Q}_c + \frac{h^3}{6} \mbf{Q}_c \mbs{\varpi}_n^{\Wdg^T} + \cdots & h \mbf{Q}_c \ebm.
\end{equation}
With $N$ large enough, $h$ will be quite small and therefore we can possibly approximate $\mbs{Q}_n$ using just the leading terms, whereupon it is a constant matrix. This approach is also appealing because it guarantees that $\mbs{Q}_n$ and hence $\mbs{Q}_{k,k-1}$ are positive semidefinite (assuming $\mbf{Q}_c$ is positive definite).  In fact, we can see $\mbs{Q}_n$ is positive definite when we keep just the first term in each block.  The downside of this approach is that we must choose $N$ large enough to ensure accuracy, which may lead to increased computational cost.

\section{Implementation}
\label{sec:fg_impl}

In this section, we sketch how to assemble our motion prior into a full estimation problem, expressed as a factor graph.  We begin by setting up the main problem where we solve for the state at all of the measurement times.  Then, we explain how to query the trajectory at an arbitrary number of times in between the measurement times after there main solve is complete.

\subsection{Main Solve}

We can simplify the notation of the linearized error in~\eqref{eq:lin_err_ccw} by writing it as 
\begin{equation}
\mbf{e}_{k,k-1} \approx \mbf{e}_{{\rm op},k,k-1} + \mbf{F}_{k,k-1} \mbs{\varepsilon}_{k-1} - \mbf{E}_{k} \mbs{\varepsilon}_k,
\end{equation}
where $\mbf{E}_k = -\mbf{K}_{k}$ and $\mbf{F}_{k,k-1} = -\mbf{K}_{k} \, \mbs{\Phi}_{k,k-1}$.  Then, the overall estimation problem can be expressed as a factor graph optimization problem as depicted in Figure~\ref{fig:ctfg}.  The large circular nodes represent the states to be estimated (at all the measurement times), while the small black circular nodes represent factors that encode probabilistic constraints between the states.  In particular, the motion prior factors connect consecutive states according to the linearized error expression above, with associated covariance $\mbf{Q}_{k,k-1}$ from \eqref{eq:Qk_conversion}.  Measurement factors  connect individual states to measurements according to the chosen measurement model, with associated measurement noise covariance.  Finally, an initial condition factor connects to the first state to anchor the trajectory estimate, with associated prior covariance.  

\begin{figure}
    \centering
    \includegraphics[width=\textwidth]{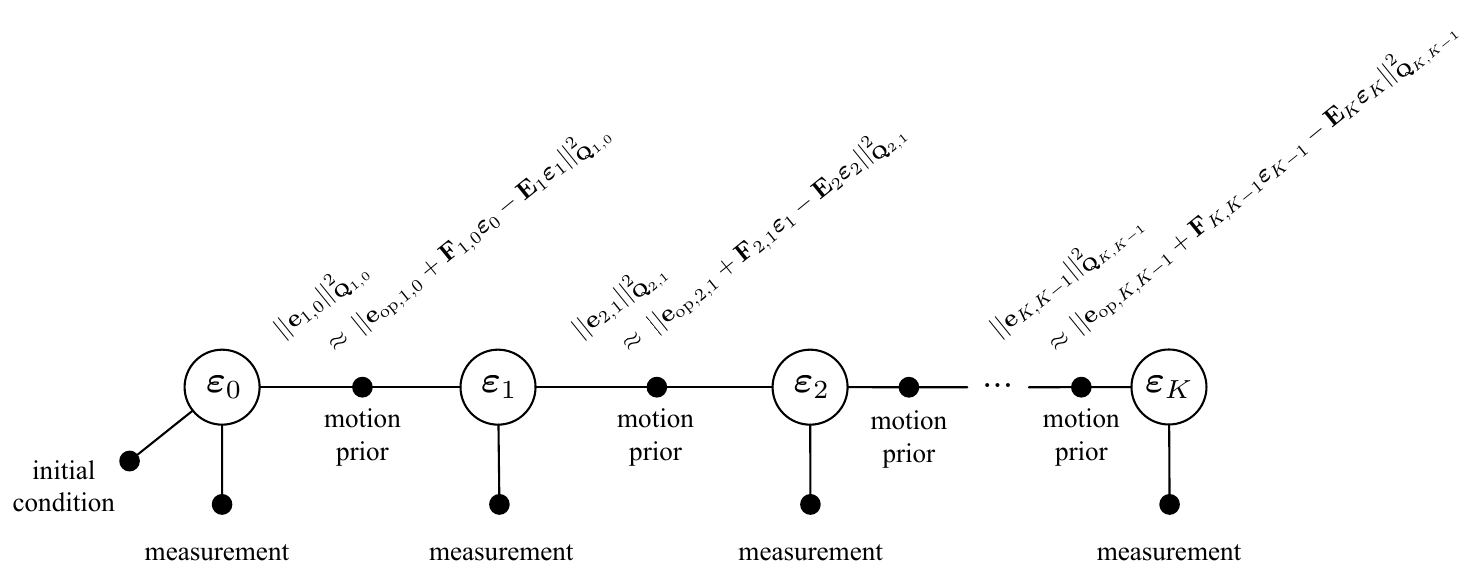}
    \caption{Factor graph representation of the \ac{GP} prior on $SE(3)$.  The large circular nodes represent the states to be estimated, while the small black circular nodes represent factors that encode probabilistic constraints between the states.}
    \label{fig:ctfg}
\end{figure}

A typical factor graph solver will solve this optimization problem iteratively, using our linearized error expressions, which are also shown as the approximations in Figure~\ref{fig:ctfg}. At each iteration, the solver finds the optimal perturbations $\mbs{\varepsilon}_k^\star = \bbm \mbs{\epsilon}_k^{\star^T} & \mbs{\eta}_k^{\star^T}   \ebm^T$ at each time step, then updates the trajectory estimate using
\begin{equation}
    \mbf{T}_{{\rm op},k} \leftarrow \mbox{Exp}\left( \mbs{\ep}^\star_k \right) \mbf{T}_{{\rm op},k}, \quad \mbs{\varpi}_{{\rm op},k} \leftarrow \mbs{\varpi}_{{\rm op},k} + \mbs{\eta}_k^\star,
\end{equation}
and then repeats the process until convergence.  At the last iteration we take our estimate to be $\est{\mbf{T}}_k = \mbf{T}_{{\rm op},k}$ and $\est{\mbs{\varpi}}_k = \mbs{\varpi}_{{\rm op},k}$.  At the end of the process, we have an estimate of the trajectory at all measurement times in terms of the mean, $\{\est{\mbf{T}}_k, \est{\mbs{\varpi}}_k\}$ and covariance, $\est{\mbf{P}}_k$, which we can also obtain from a typical factor graph solver.

\subsection{Querying}

We can also query the trajectory at time $\tau$ in between the measurement times $t_{k-1}$ and $t_k$.  To do this, we consider the factor graph in Figure~\ref{fig:gausselim}.  We insert a state at the query time, $\{\mbf{T}_\tau, \mbs{\varpi}_\tau\}$, or equivalently in the perturbation variables, $\mbs{\varepsilon}_\tau$.  This state is connected to the states at the surrounding measurement times via motion prior factors, just as in the main factor graph.  However, we then perform Gaussian elimination to remove this query state from the factor graph.  This operation results in a new binary factor connecting the two surrounding measurement times (equivalent to the original motion-prior factor between them in the main solve) and a conditional density for $\mbs{\varepsilon}_\tau$ given the surrounding states.  The mean and covariance of this conditional density can be computed using standard results from Gaussian elimination \citep{barfoot_ser24} and is given in the box in Figure~\ref{fig:gausselim}.  Then, after solving the main factor graph as before, we can use the conditional density to compute the mean and covariance of the queried state.  Importantly, we do not need to know the query time to conduct the main solve; we can perform this query operation after the main solve is complete and as many times as desired at different query times.

\begin{figure}
    \centering
    \includegraphics[width=0.9\textwidth]{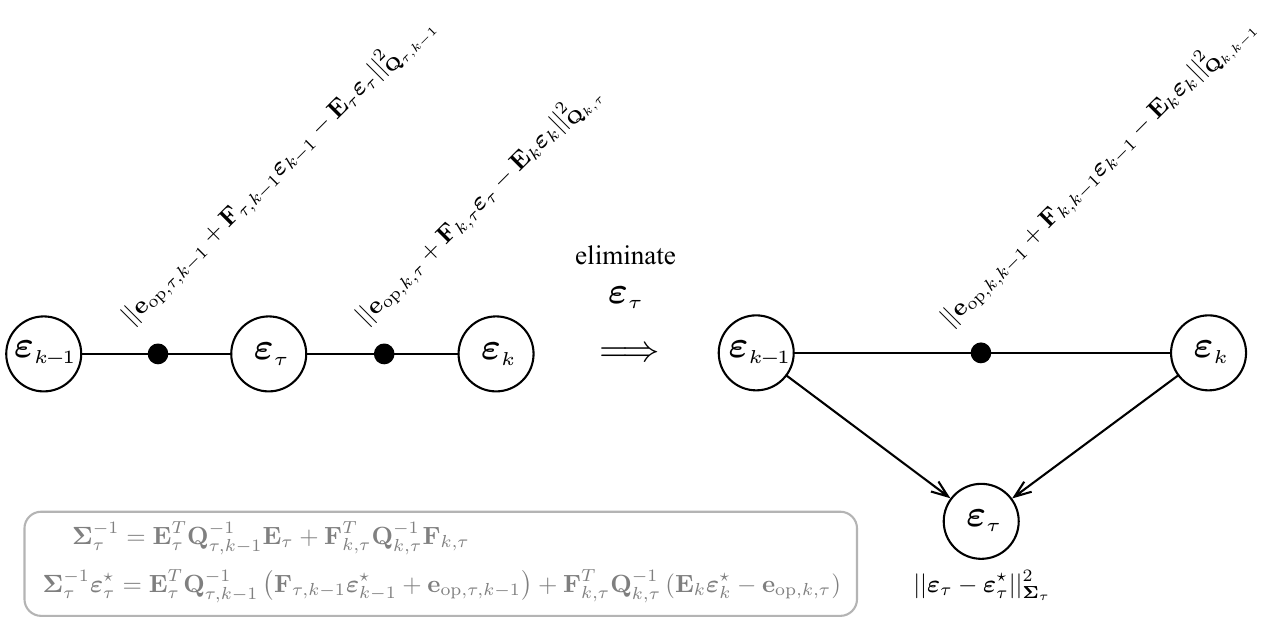}
    \caption{Factor graph representation of querying the trajectory at an arbitrary time between two measurement times.  We imagine that there was a state, $\{\mbf{T}_\tau, \mbs{\varpi}_\tau\}$, or $\mbs{\varepsilon}_\tau$ in the perturbation variables, in the original factor graph at the query time.  This state was eliminated to obtain a new binary factor (the motion prior between $k-1$ and $k$) and a conditional density for $\mbs{\varepsilon}_\tau$.  We can solve for the states at $k-1$ and $k$ in the main solve, then use the conditional density to obtain the mean and covariance of the queried state.  Moreover, we do not need to know the time of the query ahead of time; we can perform this operation after the main solve is complete and as many times as desired.}
    \label{fig:gausselim}
\end{figure}

For the mean, we look at the conditional density for $\mbs{\varepsilon}_\tau$ given in the box in Figure~\ref{fig:gausselim}.  Since the main solve will already have been conducted to convergence, the means of the perturbation variables $\mbs{\varepsilon}^\star_{k-1}$ and $\mbs{\varepsilon}^\star_{k}$ will have already converged to be close to zero.  Therefore, the mean of the perturbation at the query time, $\mbs{\varepsilon}^\star_\tau = \bbm \mbs{\ep}_{\tau}^{\star^T} & \mbs{\eta}_{\tau}^{\star^T} \ebm^T$, is simply the solution to
\begin{equation}
    \mbs{\Sigma}_\tau^{-1} \mbs{\varepsilon}^\star_\tau = \mbf{E}_\tau^T\mbf{Q}_{\tau,k-1}^{-1} \mbf{F}_{\tau,k-1}  \mbf{e}_{{\rm op},\tau,k-1} - \mbf{F}_{k,\tau}^T \mbf{Q}_{k,\tau}^{-1} \mbf{e}_{{\rm op}, k,\tau},
\end{equation}
which we carry out iteratively letting $\mbf{T}_{{\rm op},\tau} \leftarrow \mbox{Exp}\left( \mbs{\ep}^\star_{\tau} \right) \mbf{T}_{{\rm op},\tau}$ and $\mbs{\varpi}_{{\rm op},\tau} \leftarrow \mbs{\varpi}_{{\rm op},\tau} + \mbs{\eta}_{\tau}^\star$ at each iteration until convergence.  This small optimization problem for the mean is quite fast, converging in just a few iterations if we start it with a good initial guess (e.g., linear interpolation between the surrounding measurement times).

The covariance $\mbs{\Sigma}_\tau$ is given directly in the box in Figure~\ref{fig:gausselim}.  However, this only represents the covariance of the conditional density (given the states at times $k-1$ and $k$).  We need to convolve the conditional density with the marginal posterior density for the states at times $k-1$ and $k$ obtained from the main solve.  If we let the joint covariance of the states at times $k-1$ and $k$ be
\begin{equation}
\bbm \est{\mbf{P}}_{k-1} & \est{\mbf{P}}_{k,k-1}^T \\ \est{\mbf{P}}_{k,k-1} & \est{\mbf{P}}_{k} \ebm,
\end{equation}
then the covariance at the query time is given by
\begin{equation}
\est{\mbf{P}}_\tau = \mbs{\Sigma}_\tau + \bbm \mbs{\Lambda}_\tau & \mbs{\Psi}_\tau \ebm \bbm \est{\mbf{P}}_{k-1} & \est{\mbf{P}}_{k,k-1}^T \\ \est{\mbf{P}}_{k,k-1} & \est{\mbf{P}}_{k} \ebm \bbm \mbs{\Lambda}_\tau^T \\ \mbs{\Psi}_\tau^T \ebm,
\end{equation}
where
\beqn{}
\mbs{\Lambda}_\tau & = & \mbs{\Sigma}_\tau \mbf{E}_\tau^T \mbf{Q}_{\tau,k-1}^{-1} \mbf{F}_{\tau,k-1}, \\ 
\mbs{\Psi}_\tau & = & \mbs{\Sigma}_\tau \mbf{F}_{k,\tau}^T \mbf{Q}_{k,\tau}^{-1} \mbf{E}_k.
\eeqn
Our (marginal) estimate of the state at the query time $\tau$ is mean $\{\est{\mbf{T}}_\tau, \est{\mbs{\varpi}}_\tau\}$ and covariance $\est{\mbf{P}}_\tau$.  Again, we can perform this querying operation as many times as desired at different query times after the main solve is complete.  The cost of each query is $O(1)$ since it does not depend on the length of the trajectory or the number of measurement times.  

Compared to the baseline approach, the query of the mean is more involved as we must solve a small optimization problem, but the query of the covariance is simplified significantly as we avoid difficult conversions from the `local' states at the query time and can instead directly work with the perturbations to the global variables.


\section{Numerical Example}
\label{sec:numerical_example}

In this section, we will provide a numerical example to illustrate the proposed \ac{GP} prior on $SE(3)$.  We will compare different orders of the Magnus expansion and also compare to the baseline local \ac{GP} (i.e., STEAM) way of doing things \citep{barfoot_ser24}.  We will simulate a trajectory on $SE(3)$, generate noisy measurements, and then use our proposed method to estimate the trajectory.  We will evaluate the accuracy of the trajectory estimate, the quality of the interpolation between the measurement times, and the computational cost.  

Figures~\ref{fig:example_pose} and~\ref{fig:example_velocity} shows an example trajectory estimated using the proposed method with third-order Magnus expansion and 10 measurement times.  The top plot shows the trajectory on $SE(3)$, while the bottom plot shows the corresponding body-frame angular and linear velocity profiles over time.  All of the methods that we will compare look quite similar to this plot and we therefore do not show additional plots of the trajectories for the other methods.

\begin{figure}[h]
    \centering
    \includegraphics[width=\textwidth]{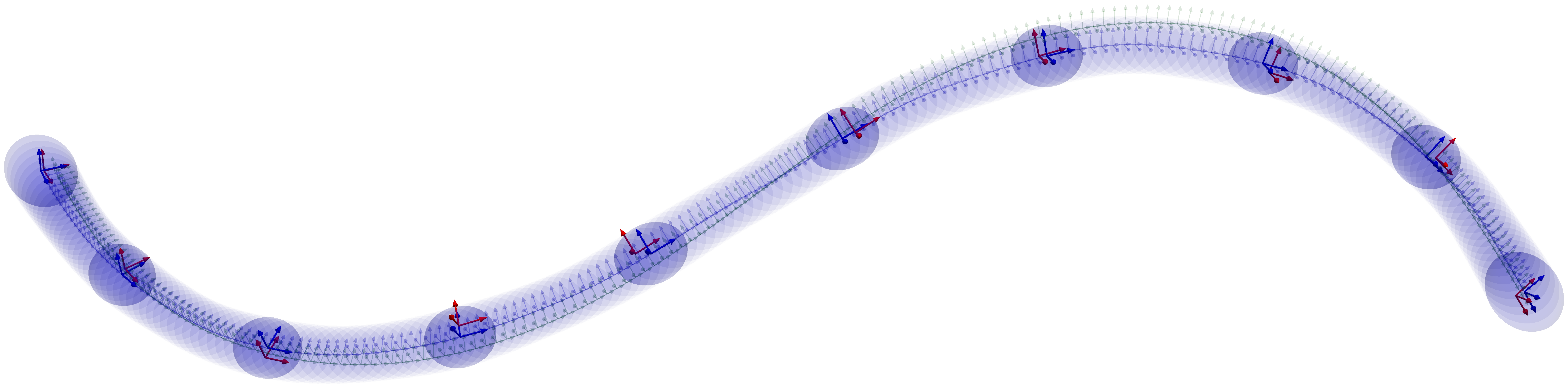}
    \caption{{\em Pose Estimate:} Example trajectory on $SE(3)$ estimated using the proposed \ac{GP} prior with third-order Magnus expansion and 10 measurement times.  The large dark-blue frames are the estimated poses, the smaller light-blue frames are interpolated poses after the main solve, the dark-blue ellipsoids are the marginal covariances at the estimated times, the light-blue ellipsoids are the interpolated covariances, the red frames are the noisy pose measurements, and the green frames are the ground-truth trajectory.}
    \label{fig:example_pose}
\end{figure}

\begin{figure}[t]
    \centering
    \includegraphics[width=\textwidth]{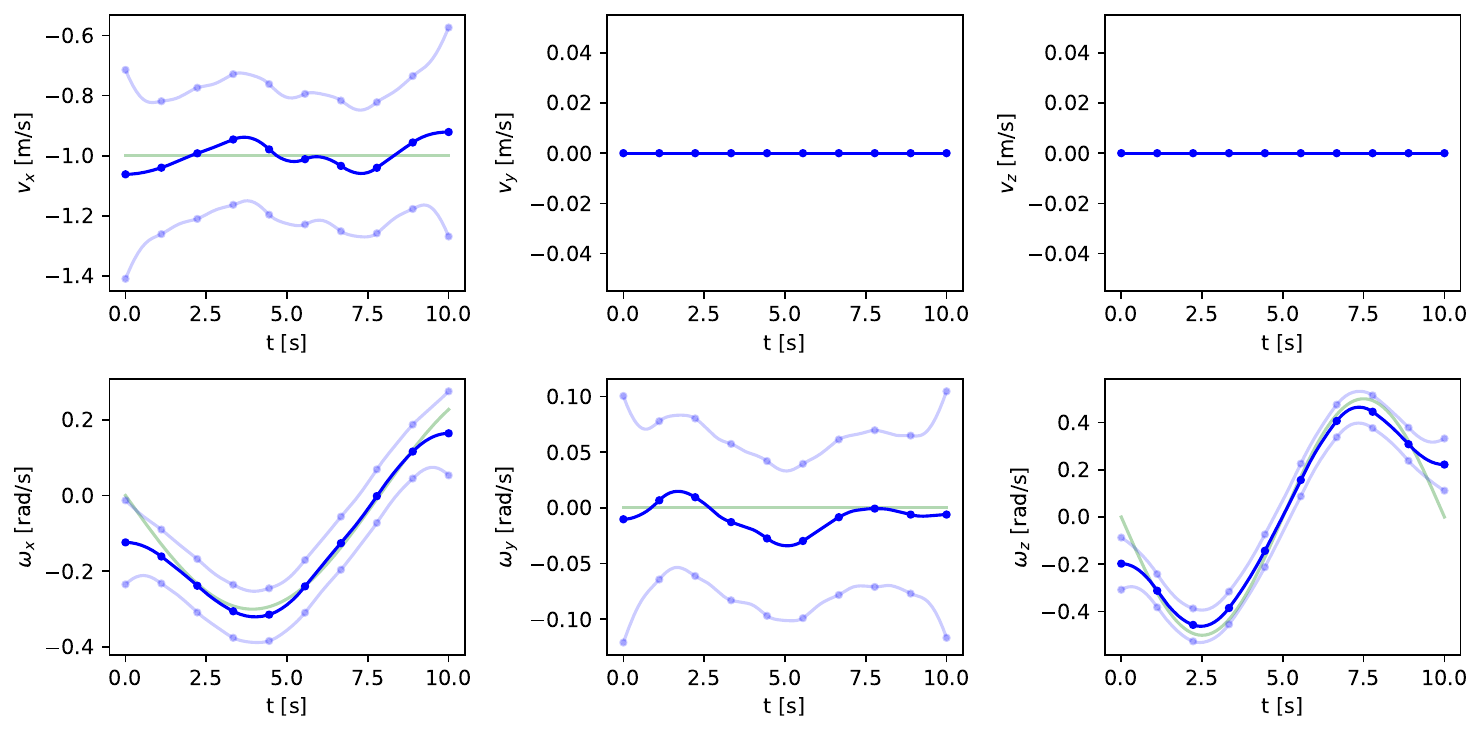}
    \caption{{\em Velocity Estimate:} Corresponding body-frame angular velocity (top row) and linear velocity (bottom row) profiles over time.  The dark-blue lines are the mean velocity estimate (dots for estimated time, lines for interpolation after the main solve), the light-blue dots and lines represent the covariances, and the green line shows the ground-truth velocity profile.  There are no direct measurements of velocity; these are inferred from the pose measurements and the motion prior.}
    \label{fig:example_velocity}
\end{figure}

\subsection{Estimation Accuracy}

To evaluate the accuracy of the trajectory estimate, we ran a large number of trials of the scenario depicted in Figure~\ref{fig:example_pose} with different flavours (i.e., 1, 2, 3 terms in the Magnus expansion) of our algorithm as well as the baseline.  We computed the root mean square error (RMSE) between the estimated trajectory and the ground-truth trajectory at a large number of evenly spaced times over the entire trajectory duration and varied the number of times at which we received noisy pose measurements from $K_{\rm meas} = 3\ldots15$.  At each value of $K_{\rm meas}$, we ran 50 trials with different random seeds to generate different noise realizations.  Figures~\ref{fig:ground_truth_errors_pose} and~\ref{fig:ground_truth_errors_velocity} show the RMSE results for the pose and velocity estimates, respectively.  When the spacing between measurements is large, the proposed methods do slightly better and as the number of measurement times increases, all methods converge to similar performance.

\begin{figure}[p]
    \centering
    \includegraphics[width=0.9\textwidth]{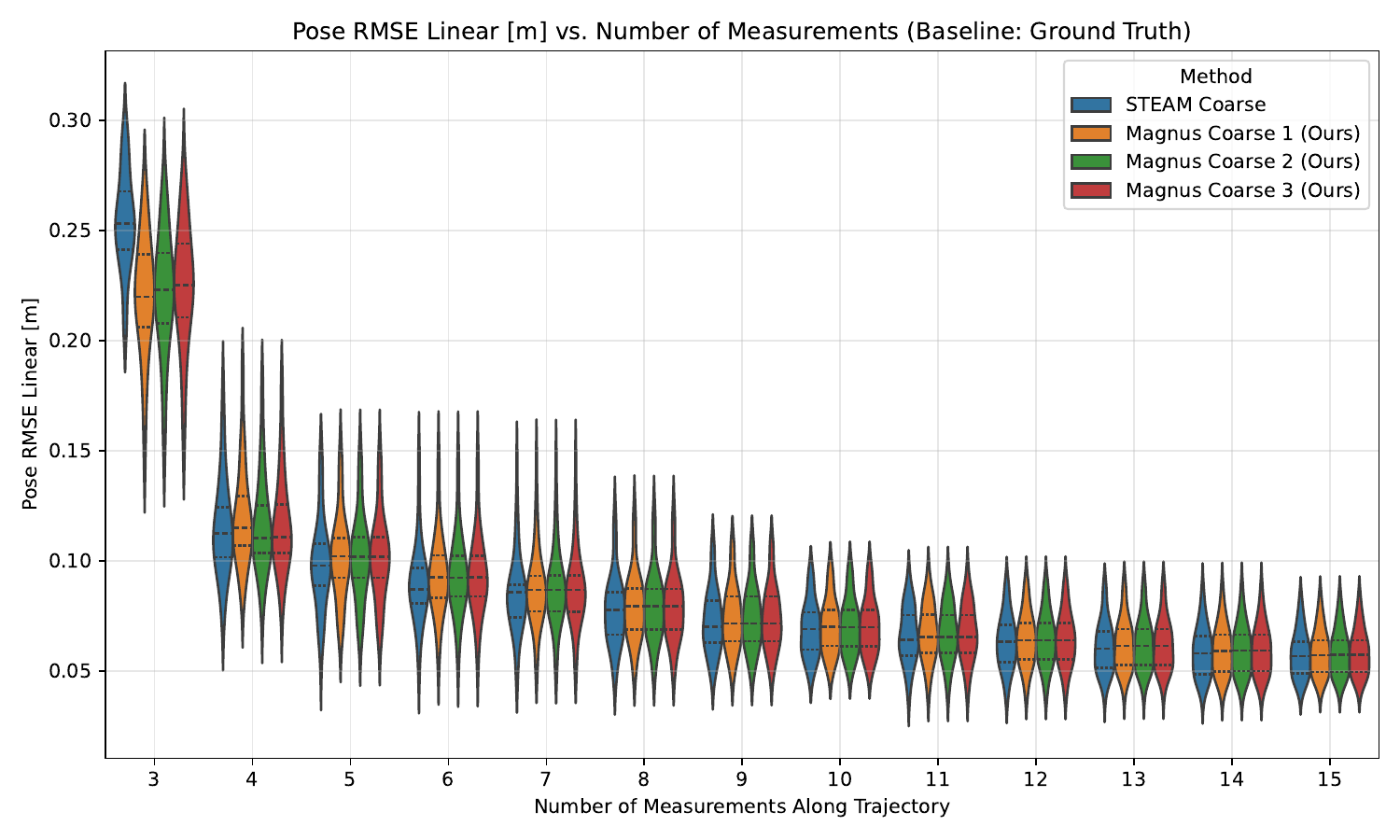}
    \includegraphics[width=0.9\textwidth]{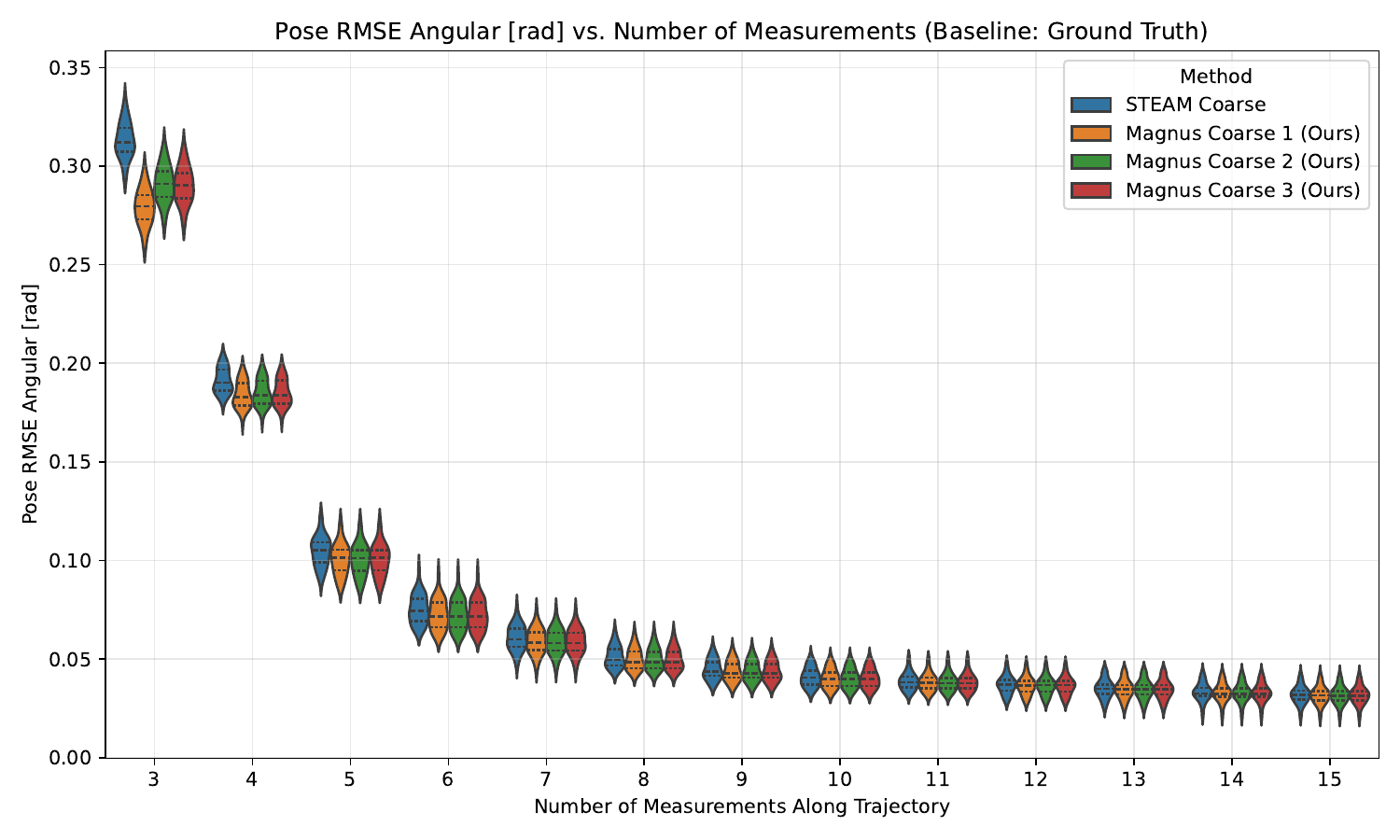}
    \caption{{\em Pose Accuracy:} Plots show the RMSE errors for the (top) linear and (bottom) angular components of the pose when compared to the ground-truth trajectory at the measurement times.  The different colours represent the different methods: baseline (STEAM) in blue, proposed method with 1st-order Magnus expansion in orange, 2nd-order in green, and 3rd-order in red.  The $x$-axis shows the number of measurement times used in the estimation.  Each violin plot summarizes the distribution of RMSE values over 50 trials with different noise realizations.}
    \label{fig:ground_truth_errors_pose}
\end{figure}

\begin{figure}[p]
    \centering
    \includegraphics[width=0.9\textwidth]{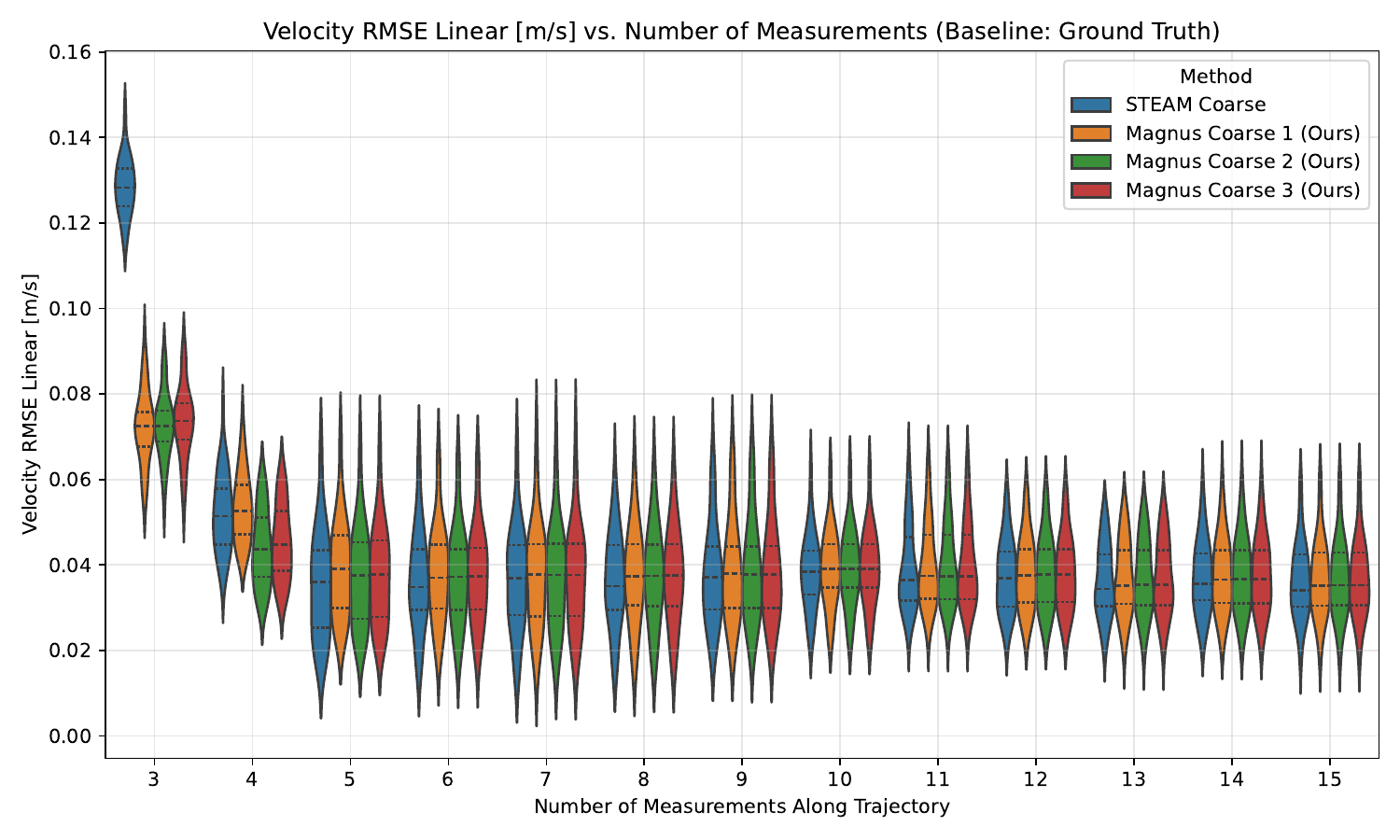}
    \includegraphics[width=0.9\textwidth]{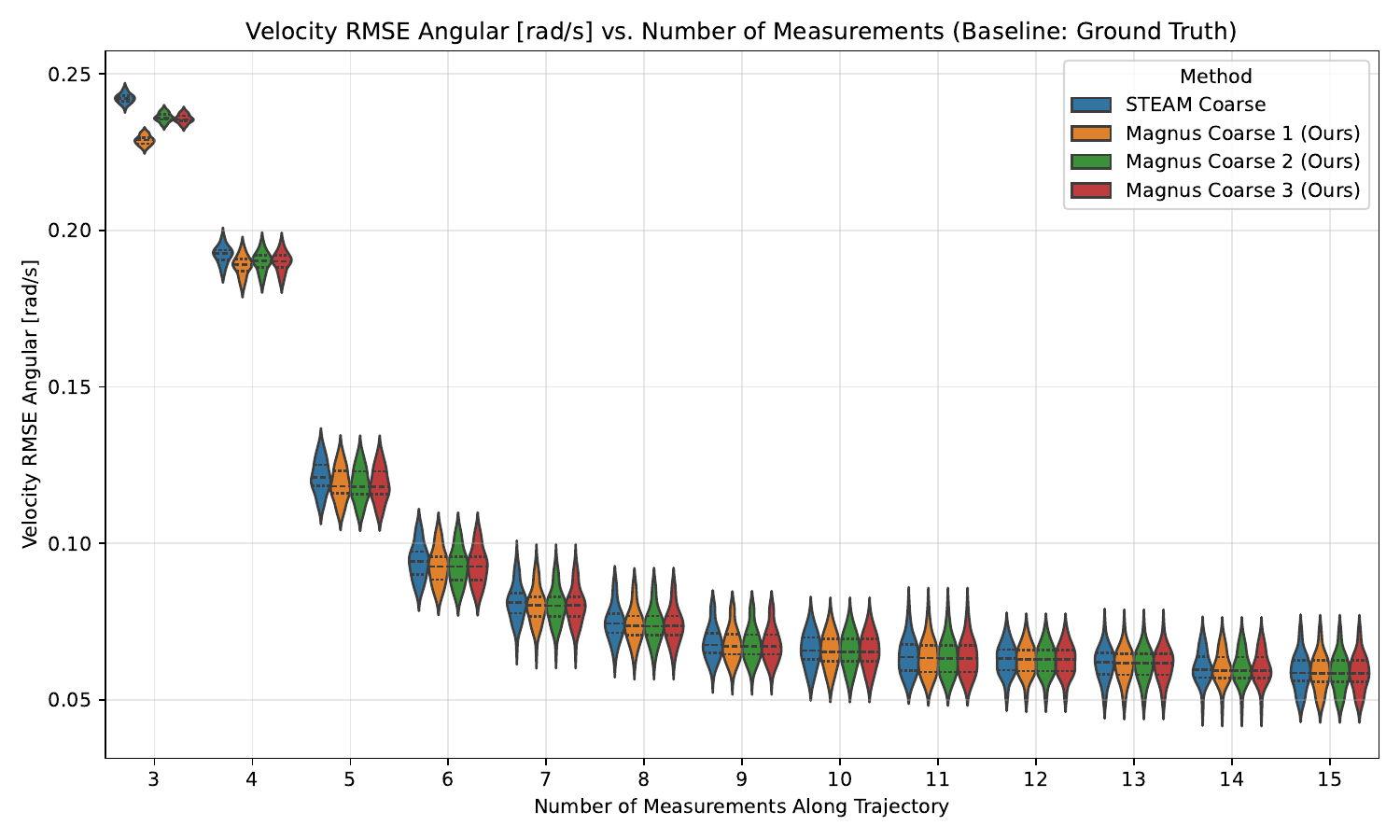}
    \caption{{\em Velocity Accuracy:} Plots show the RMSE errors for the (top) linear and (bottom) angular components of the velocity when compared to the ground-truth trajectory at the measurement times.  The different colours represent the different methods: baseline (STEAM) in blue, proposed method with 1st-order Magnus expansion in orange, 2nd-order in green, and 3rd-order in red.  The $x$-axis shows the number of measurement times used in the estimation.  Each violin plot summarizes the distribution of RMSE values over 50 trials with different noise realizations.}
    \label{fig:ground_truth_errors_velocity}
\end{figure}

\subsection{Interpolation Quality}

To evaluate the quality of the interpolation between measurement times, we again ran a large number of trials of the scenario depicted in Figure~\ref{fig:example_pose} with different flavours (i.e., 1, 2, 3 terms in the Magnus expansion) of our algorithm as well as the baseline.  We computed the RMSE between the interpolated trajectory (comprising $20 \times K_{\rm meas}$ timestamps) and the baseline method run with an equal number of estimation times (i.e., `STEAM Fine').  The idea here was to use the `Fine' method to represent a high-quality estimate of the trajectory against which we could compare the `Coarse' interpolated estimates.  We again varied the number of times at which we received noisy pose measurements from $K_{\rm meas} = 3\ldots15$.  At each value of $K_{\rm meas}$, we ran 50 trials with different random seeds to generate different noise realizations.  Figures~\ref{fig:interpolation_errors_pose} and~\ref{fig:interpolation_errors_velocity} show the RMSE results for the pose and velocity estimates, respectively.  When the spacing between measurements is large, the baseline method does better on the linear state variables, while the proposed methods do slightly better on the angular variables.  As the number of measurement times increases, all methods converge to similar performance.

\begin{figure}[p]
    \centering
    \includegraphics[width=0.9\textwidth]{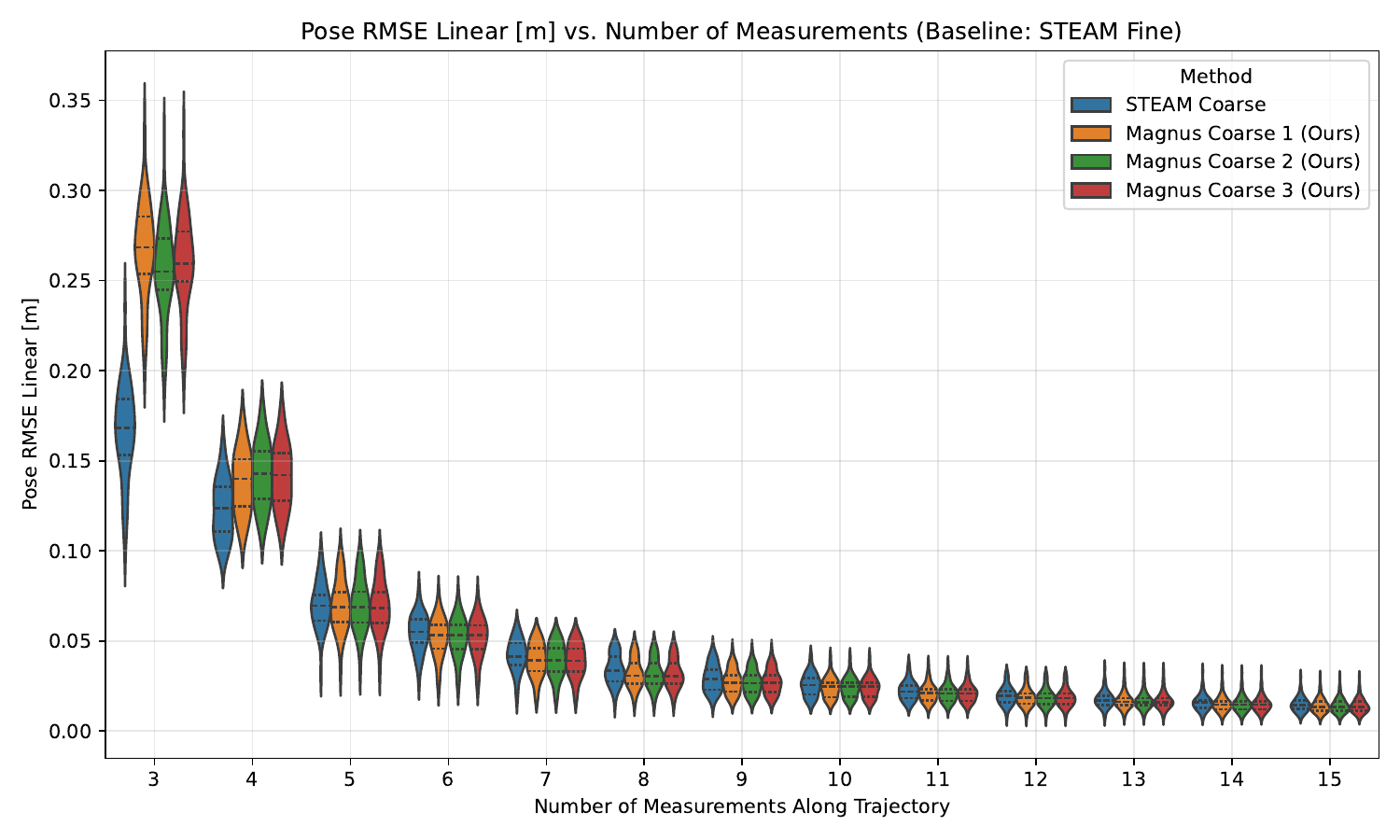}
    \includegraphics[width=0.9\textwidth]{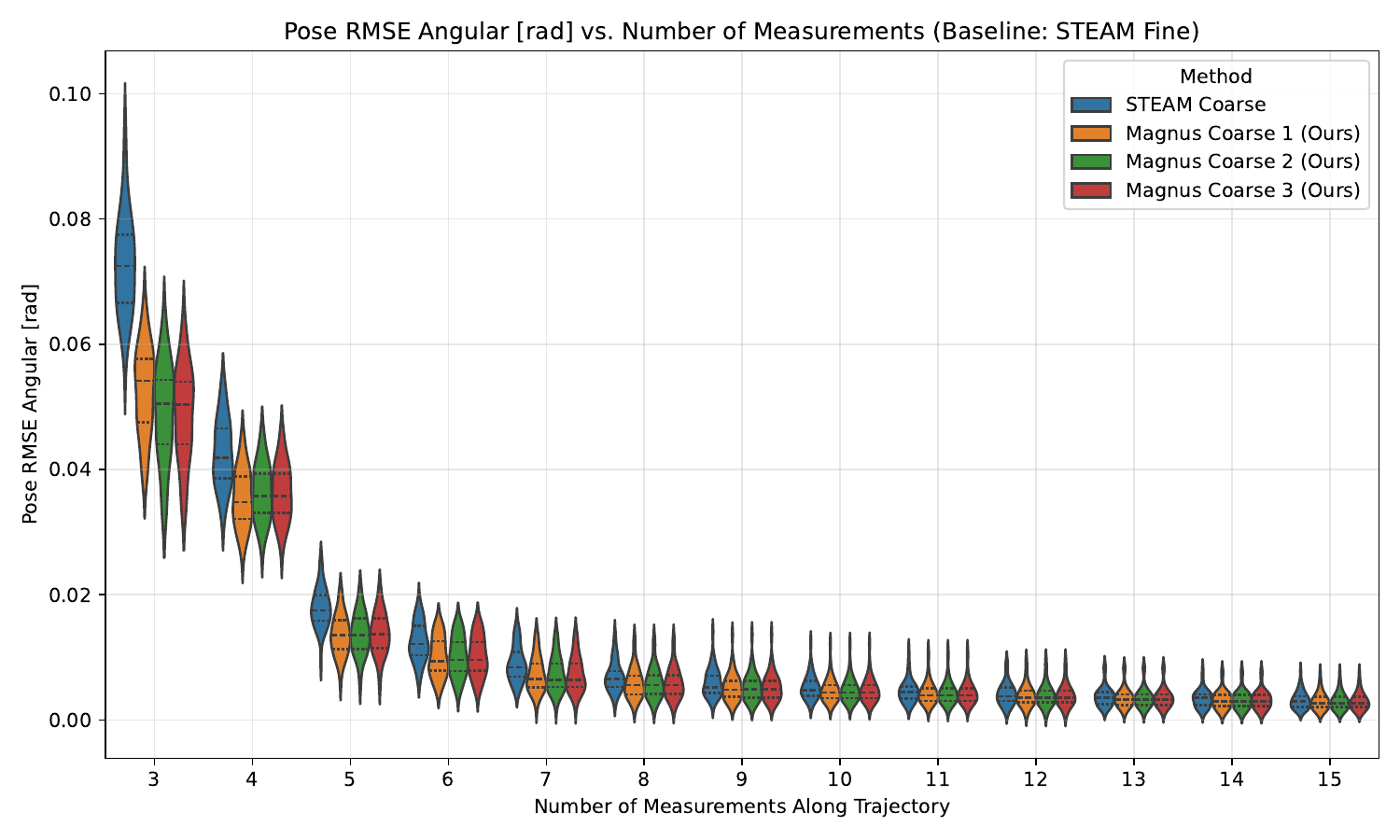}
    \caption{{\em Pose Interpolation Quality:} Plots show the RMSE errors for the (top) linear and (bottom) angular components of the pose when compared to the `STEAM Fine' trajectory at the interpolation times.  The different colours represent the different methods: baseline (STEAM) in blue, proposed method with 1st-order Magnus expansion in orange, 2nd-order in green, and 3rd-order in red.  The $x$-axis shows the number of measurement times used in the estimation.  Each violin plot summarizes the distribution of RMSE values over 50 trials with different noise realizations.}
    \label{fig:interpolation_errors_pose}
\end{figure}

\begin{figure}[p]
    \centering
    \includegraphics[width=0.9\textwidth]{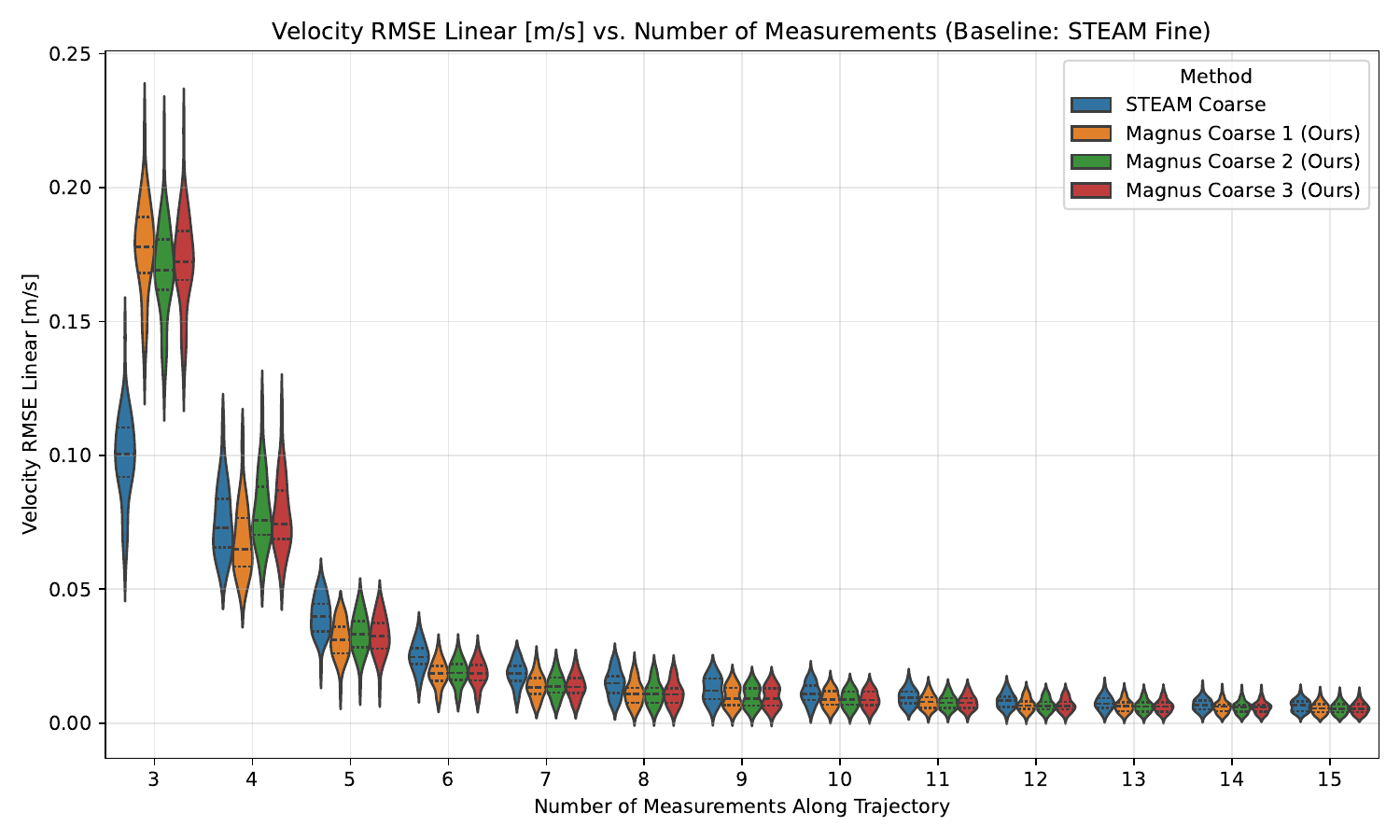}
    \includegraphics[width=0.9\textwidth]{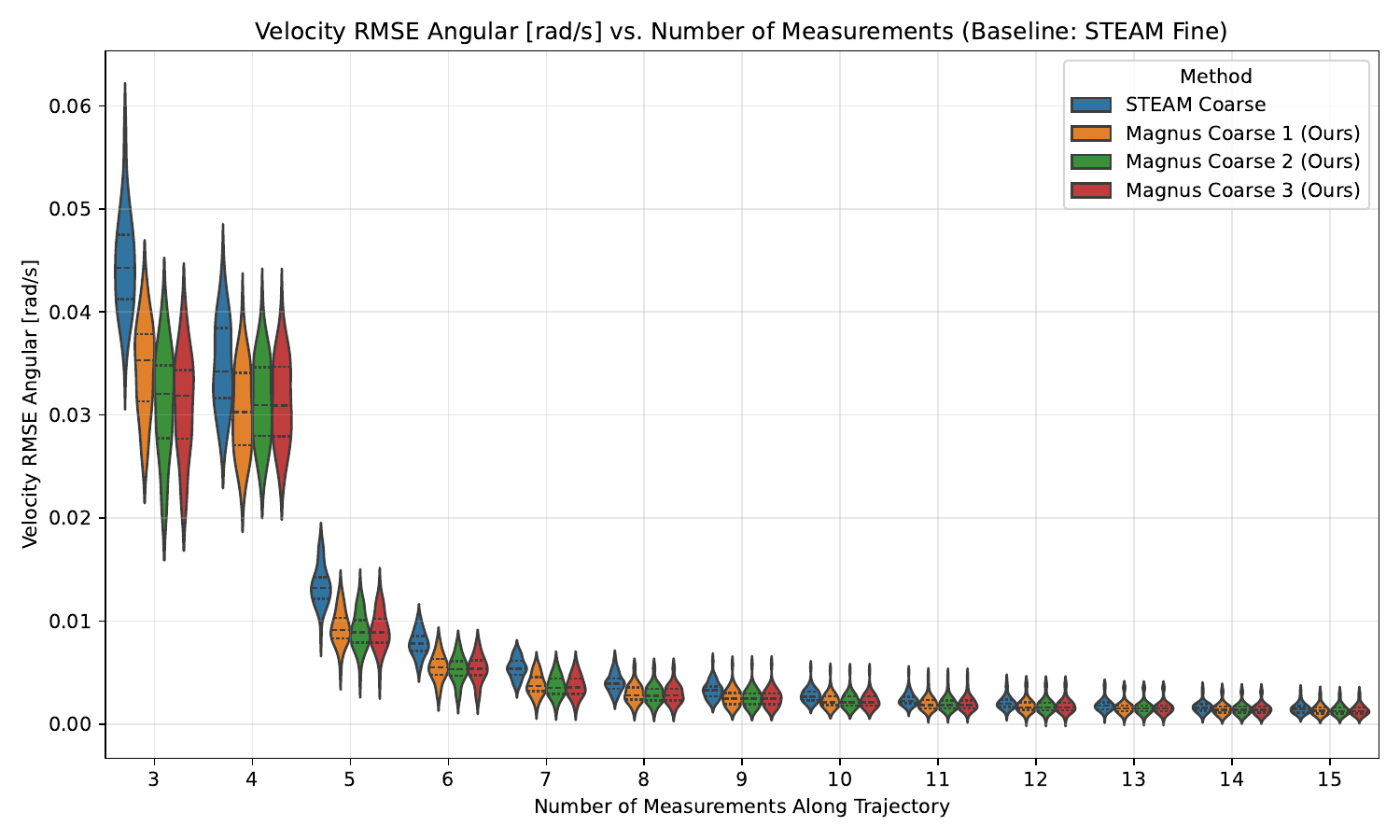}
    \caption{{\em Velocity Interpolation Quality:} Plots show the RMSE errors for the (top) linear and (bottom) angular components of the velocity when compared to the `STEAM Fine' trajectory at the interpolation times.  The different colours represent the different methods: baseline (STEAM) in blue, proposed method with 1st-order Magnus expansion in orange, 2nd-order in green, and 3rd-order in red.  The $x$-axis shows the number of measurement times used in the estimation.  Each violin plot summarizes the distribution of RMSE values over 50 trials with different noise realizations.}
    \label{fig:interpolation_errors_velocity}
\end{figure}

\subsection{Computational Cost}

We also measured the computational cost of each method as a function of the number of measurement times used in the estimation.  We measured both the main solve time as well as the interpolation time (to interpolate both the mean and covariance at $20 \times K_{\rm meas}$ timestamps after the main solve).  Figures~\ref{fig:computational_cost} shows the computation time results.  We can see that the baseline method is much faster for both the main solve and interpolation times.  Among the proposed methods, we see that the first-order Magnus expansion is fastest, followed by second-order and then third-order, as expected.

\begin{figure}[p]
    \centering
    \includegraphics[width=0.9\textwidth]{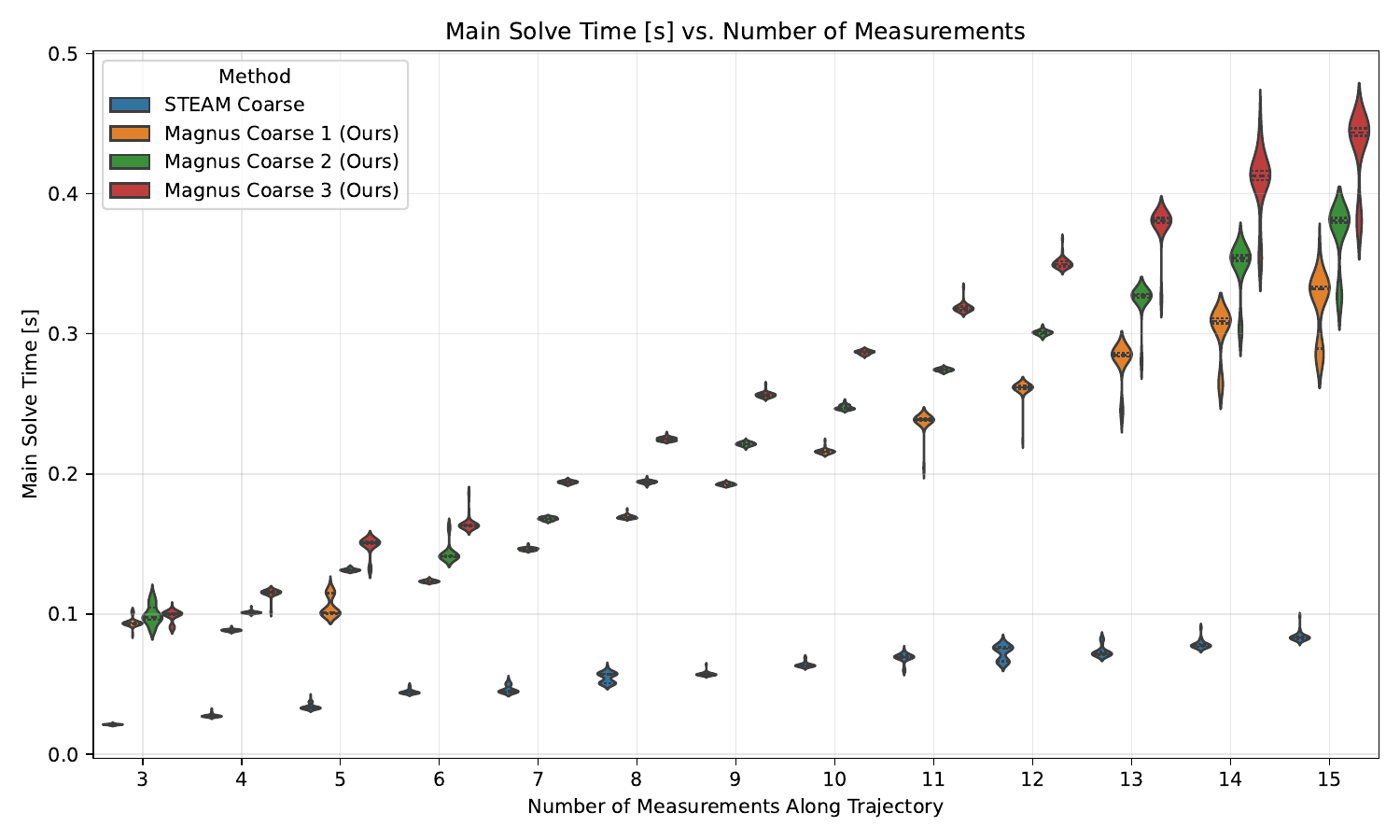}
    \includegraphics[width=0.9\textwidth]{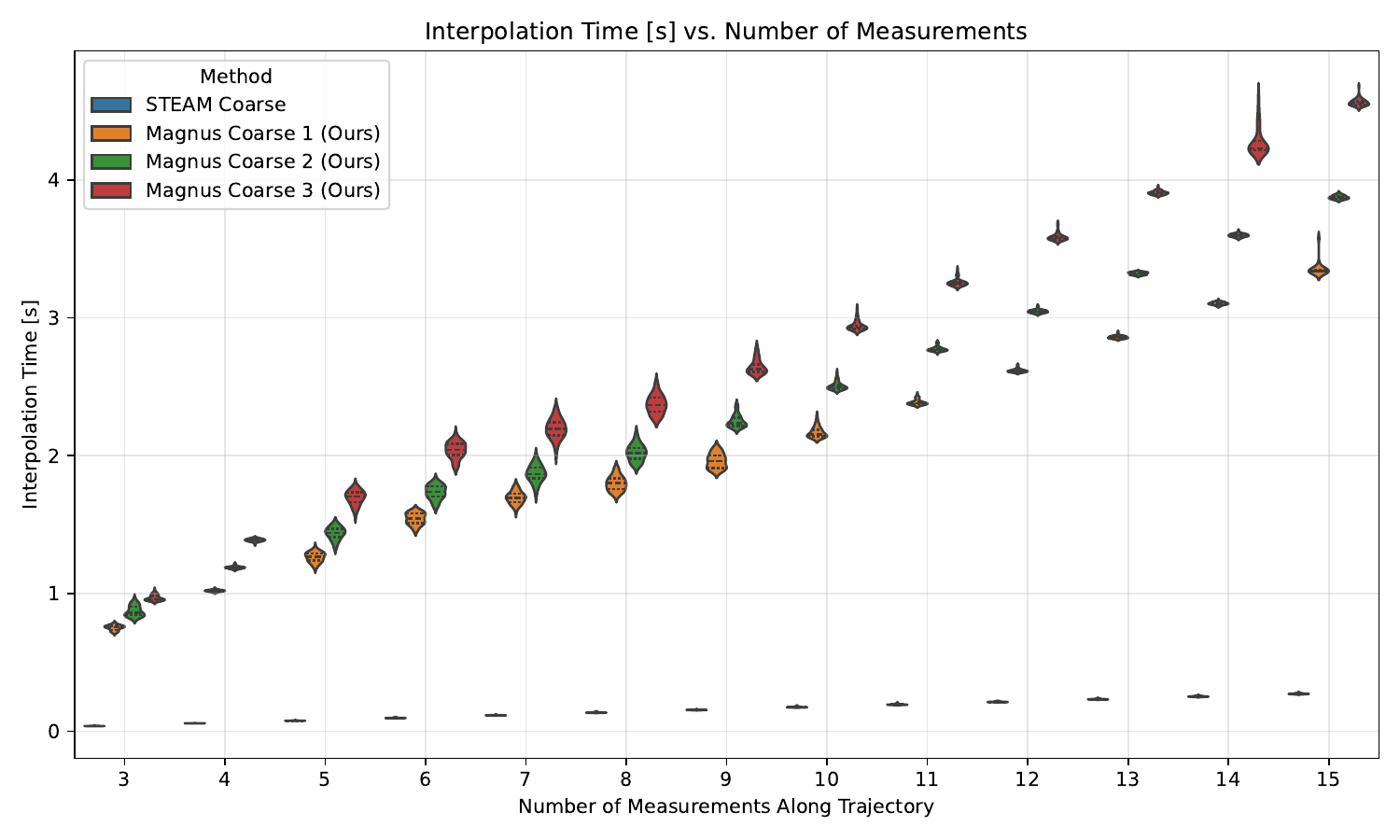}
    \caption{{\em Computational Cost:} We measured the computational cost of the main solve time (top) and interpolation time (bottom) as a function of the number of measurement times used in the estimation. The different colours represent the different methods: baseline (STEAM) in blue, proposed method with 1st-order Magnus expansion in orange, 2nd-order in green, and 3rd-order in red.  Each violin plot summarizes the distribution of computation times over 50 trials with different noise realizations.}
    \label{fig:computational_cost}
\end{figure}


\section{Conclusion}
\label{sec:conclusion}

We have developed a continuous-time \ac{GP} motion prior on $SE(3)$ that is capable of representing complex trajectories.  We used the Magnus expansion to compute the state transition matrix required to propagate the state and compute the discrete-time process noise covariance.  We integrated this motion prior into a framework for trajectory estimation and showed how to query the trajectory at arbitrary times after the main solve is complete.  

When comparing to the baseline method that we originally sought to improve upon, we found that the performance of our new method was quite similar in terms of accuracy and interpolation quality, but was computationally much slower owing to the more complex state transition matrix and process noise covariance calculations.  Future work could explore ways to speed up these calculations, perhaps by precomputing certain terms or finding approximations that are faster to compute.  Another avenue for future work could be to explore different motion models or priors that might better capture the dynamics of certain systems.  For now, our conclusion is to stick with baseline method but hope that some of the ideas presented here might be useful in other contexts.




\appendix

\section{Magnus Expansion}
\label{app:magnus}

We follow the treatment of \citet{blanes09} in this section.
When solving an \ac{LTV} \ac{SDE} of the form in \eqref{eq:ltv_sde}, we require the state transition matrix $\mbs{\Phi}(t,s)$ that propagates the state from time $s$ to time $t$.  This matrix satisfies the matrix differential equation
\begin{equation}\label{eq:stm_de}
    \dot{\mbs{\Phi}}(t,s) = \mbf{A}(t) \mbs{\Phi}(t,s), \quad \mbs{\Phi}(s,s) = \mbf{1},
\end{equation}
where $\dot{(\cdot)}$ denotes differentiation with respect to $t$.
The Magnus expansion proposes that the solution can be written as 
\begin{equation}\label{eq:magnus1}
    \mbs{\Phi}(t,s) = \exp\left( \mbs{\Omega}(t,s) \right),
\end{equation}
where $\mbs{\Omega}(t,s)$.  Using the properties of the matrix exponential and its derivative, we can write that 
\begin{equation}
    \dot{\mbs{\Phi}}(t,s) = \int_0^1 \exp\left( \alpha \mbs{\Omega}(t,s) \right) \dot{\mbs{\Omega}}(t,s) \exp\left( (1-\alpha) \mbs{\Omega}(t,s) \right) d\alpha.
\end{equation}
Post-multiplying by $\mbs{\Phi}(t,s)^{-1} = \exp\left( -\mbs{\Omega}(t,s) \right)$ and using~\eqref{eq:stm_de}, we have
\begin{equation}
    \mbf{A}(t) = \int_0^1 \underbrace{\exp\left( \alpha \mbs{\Omega}(t,s) \right) \dot{\mbs{\Omega}}(t,s) \exp\left( -\alpha \mbs{\Omega}(t,s) \right)}_{\mbox{Ad$_{\alpha\mbs{\Omega}}\dot{\mbs{\Omega}}$}} d\alpha,
\end{equation}
where the adjoint operators are defined as
\begin{gather}
    \mbox{Ad}_{\alpha\mbs{\Omega}} \dot{\mbs{\Omega}} = \exp\left( \alpha \mbs{\Omega} \right) \dot{\mbs{\Omega}} \exp\left( -\alpha \mbs{\Omega} \right) = \sum_{n=0}^\infty \frac{\alpha^n}{n!} \mbox{ad}_{\mbs{\Omega}}^n \dot{\mbs{\Omega}},  \\ \mbox{ad}_{\mbs{\Omega}}^n \dot{\mbs{\Omega}} = [\mbs{\Omega}, \mbox{ad}_{\mbs{\Omega}}^{n-1} \dot{\mbs{\Omega}}], \quad \mbox{ad}_{\mbs{\Omega}}^0 \dot{\mbs{\Omega}} = \dot{\mbs{\Omega}}.
\end{gather}
After performing the integral over $\alpha$, we then have that 
\begin{equation}
    \mbf{A}(t) = \sum_{n=0}^\infty \frac{1}{(n+1)!} \mbox{ad}_{\mbs{\Omega}(t,s)}^n \dot{\mbs{\Omega}}(t,s) = \mathcal{J}(\dot{\mbs{\Omega}}(t,s)),
\end{equation}
where the right-hand side represents an operator, $\mathcal{J}$, acting on $\dot{\mbs{\Omega}}(t,s)$.
Inverting this operator relationship leads to the following expression for $\dot{\mbs{\Omega}}(t,s)$:
\begin{equation}
    \dot{\mbs{\Omega}}(t,s) = \mathcal{J}^{-1}(\mbf{A}(t)) = \sum_{n=0}^\infty \frac{B_n}{n!} \mbox{ad}_{\mbs{\Omega}(t,s)}^n \mbf{A}(t),
\end{equation}
where $B_n$ are the Bernoulli numbers. Thus, we want to find the matrix $\mbs{\Omega}(t,s)$ that satisfies this differential equation, which is done typically using Picard iteration.

In our case, we are primarily interested in the situation where
\begin{equation}
    \mbf{A}(\tau) = \mbf{A} + \tau \mbf{B}
\end{equation}
where $\tau$ is time. We will later change variables to a slightly different form, but this form makes the integrals easiest.  We can then use a Picard iteration to solve for $\mbs{\Omega}(\tau,0)$ as follows:
\begin{equation}
    \mbs{\Omega}^{(n+1)}(\tau,0) = \int_{0}^\tau \sum_{n=0}^\infty \frac{B_n}{n!} \mbox{ad}_{\mbs{\Omega}^{(n)}(s,0)}^n \mbf{A}(s) ds,
\end{equation}
starting with $\mbs{\Omega}^{(0)}(\tau,0) = \mbf{0}$. Each time we solve for $\mbs{\Omega}^{(n+1)}(\tau,0)$, out to a desired order in $\tau$, then substitute this back into the right-hand side to get the next iteration; each time we substitute, the order of $\tau$ increases and we must be careful to include all the contributing terms from the previous iteration.  After several iterations we can identify and group different powers of $\tau$ to obtain the Magnus expansion term by term.  The first four terms are given by
\begin{subequations}
\begin{align}
    \mbs{\Omega}_1(\tau,0) &= \tau \mbf{A} + \frac{\tau^2}{2} \mbf{B}, \\
    \mbs{\Omega}_2(\tau,0) &= \frac{\tau^3}{12} [ \mbf{B}, \mbf{A}] \\
    \mbs{\Omega}_3(\tau,0) &= \frac{\tau^5}{240} [\mbf{B}, [\mbf{B}, \mbf{A}]], \\
    \mbs{\Omega}_4(\tau,0) &= -\frac{\tau^5}{720} [\mbf{A}, [\mbf{A}, [\mbf{B}, \mbf{A}]]] - \frac{\tau^6}{720} [\mbf{B}, [\mbf{A}, [\mbf{B}, \mbf{A}]]] - \frac{\tau^7}{5040} [\mbf{B}, [\mbf{B}, [\mbf{B}, \mbf{A}]]],
\end{align}
\end{subequations}
which we see all involve right-nested commutators with $[\mbf{B}, \mbf{A}]$, starting from the second term.  It is worth noting that by the Jacobi identity we know that 
\begin{equation}
    [\mbf{B}, [\mbf{A}, [\mbf{B}, \mbf{A}]]] = [\mbf{A}, [\mbf{B}, [\mbf{B}, \mbf{A}]]],
\end{equation}
which allowed us to merge two terms in $\mbs{\Omega}_4(\tau,0)$ above.

We will now make two different substitutions to produce the forms we need in the main text.  First, we will require $\mbs{\Phi}(t,t_{k-1}) = \exp\left( \mbs{\Omega}(t,t_{k-1})\right)$.  For this we let 
\begin{equation}
    \tau = t - t_{k-1}, \quad \mbf{A} = \mbf{A}_{k-1}, \quad \mbf{B} = \frac{1}{\Delta t_k}(\mbf{A}_k - \mbf{A}_{k-1}),
\end{equation}
where $\Delta t_k = t_k - t_{k-1}$.  This leads to 
\begin{equation}
    \mbf{A}(t) = \mbf{A}_{k-1} + \frac{t - t_{k-1}}{\Delta t_k} (\mbf{A}_k - \mbf{A}_{k-1}),
\end{equation}
and then the first four terms of the Magnus expansion become
\begin{subequations}
\begin{align}
    \mbs{\Omega}_1(t,t_{k-1}) &= (t - t_{k-1}) \mbf{A}_{k-1} + \frac{(t - t_{k-1})^2}{2 \, \Delta t_k} (\mbf{A}_k - \mbf{A}_{k-1}), \\
    \mbs{\Omega}_2(t,t_{k-1}) &= \frac{(t - t_{k-1})^3}{12 \, \Delta t_k} [\mbf{A}_k, \mbf{A}_{k-1}], \\
    \mbs{\Omega}_3(t,t_{k-1}) &= \frac{(t - t_{k-1})^5}{240 \, \Delta t_k^2} [\mbf{A}_k - \mbf{A}_{k-1}, [\mbf{A}_k, \mbf{A}_{k-1}]], \\
    \mbs{\Omega}_4(t,t_{k-1}) &= -\frac{(t - t_{k-1})^5}{720 \, \Delta t_k^2} [\mbf{A}_{k-1}, [\mbf{A}_{k-1}, [\mbf{A}_k, \mbf{A}_{k-1}]]] - \frac{(t - t_{k-1})^6}{720 \, \Delta t_k^3} [\mbf{A}_k - \mbf{A}_{k-1}, [\mbf{A}_{k-1}, [\mbf{A}_k, \mbf{A}_{k-1}]]] \nonumber \\
    & \qquad - \frac{(t - t_{k-1})^7}{5040 \, \Delta t_k^3} [\mbf{A}_k - \mbf{A}_{k-1}, [\mbf{A}_k - \mbf{A}_{k-1}, [\mbf{A}_k, \mbf{A}_{k-1}]]].
\end{align}
\end{subequations}
If we then take $t=t_k$, the first four terms are given in the main text as~\eqref{eq:magnus_terms} where $\mbs{\Omega}_{k,k-1} = \mbs{\Omega}(t_k,t_{k-1})$ as a shorthand.

Second, at times we may require $\mbs{\Phi}(t_k, t) = \exp\left( \mbs{\Omega}(t_k, t)\right)$.  This is a bit subtle since the time indices are now reversed, but we can still use the same form of the Magnus expansion.  We note that $\mbs{\Phi}(t_k, t) = \mbs{\Phi}(t, t_k)^{-1} = \exp\left( -\mbs{\Omega}(t, t_k)\right)$ so that $\mbs{\Omega}(t_k, t) = -\mbs{\Omega}(t, t_k)$.  

To build the terms of $\mbs{\Omega}(t, t_k)$ we let
\begin{equation}
    \tau = t - t_k, \quad \mbf{A} = \mbf{A}_k, \quad \mbf{B} = \frac{-1}{\Delta t_k}(\mbf{A}_{k-1} - \mbf{A}_k),
\end{equation}
which leads to
\begin{equation}
    \mbf{A}(t) = \mbf{A}_k + \frac{t_k - t}{\Delta t_k} (\mbf{A}_{k-1} - \mbf{A}_k).
\end{equation}
The first four terms of the Magnus expansion then become
\begin{subequations}
\begin{align}
    \mbs{\Omega}_1(t_k,t) &=  (t_k - t) \mbf{A}_k + \frac{(t_k - t)^2}{2 \, \Delta t_k} (\mbf{A}_{k-1} - \mbf{A}_k), \\
    \mbs{\Omega}_2(t_k,t) &= \frac{(t_k - t)^3}{12 \, \Delta t_k} [\mbf{A}_{k}, \mbf{A}_{k-1}], \\        
    \mbs{\Omega}_3(t_k,t) &= \frac{(t_k - t)^5}{240 \, \Delta t_k^2} [\mbf{A}_{k} - \mbf{A}_{k-1}, [\mbf{A}_{k}, \mbf{A}_{k-1}]], \\
    \mbs{\Omega}_4(t_k,t) &= -\frac{(t_k - t)^5}{720 \, \Delta t_k^2} [\mbf{A}_k, [\mbf{A}_k, [\mbf{A}_{k}, \mbf{A}_{k-1}]]] + \frac{(t_k - t)^6}{720 \, \Delta t_k^3} [\mbf{A}_{k} - \mbf{A}_{k-1}, [\mbf{A}_k, [\mbf{A}_{k}, \mbf{A}_{k-1}]]] \nonumber \\
    & \qquad - \frac{(t_k - t)^7}{5040 \, \Delta t_k^3} [\mbf{A}_{k} - \mbf{A}_{k-1}, [\mbf{A}_{k} - \mbf{A}_{k-1}, [\mbf{A}_{k}, \mbf{A}_{k-1}]]].
\end{align}
\end{subequations}
Again, if we take $t=t_{k-1}$, the first four terms are given in the main text as~\eqref{eq:magnus_terms} where $\mbs{\Omega}(t_k,t_{k-1})$ is again written as $\mbs{\Omega}_{k,k-1}$ for brevity.

\section{Proof of Discrete-Time State Transition Equivalence}
\label{app:commdiag_proof}

To show the equivalence of travelling both directions around the commutative diagram in Figure~\ref{fig:commdiag}, we need to show that the discrete-time state transition matrix obtained from the Magnus expansion applied to the \ac{LTV} \ac{SDE} in the perturbations is the same as that obtained from the Magnus expansion applied to the original \ac{LTV} \ac{SDE} in \eqref{eq:lie_sde}.  In detail, we need the following theorem:

\begin{theorem}
Let $\mbs{\psi} = \mbs{\psi}_1 + \mbs{\psi}_2 + \cdots + \mbs{\psi}_N$ be the `Magnus vector' obtained by applying the Magnus expansion to the \ac{LTV} \ac{SDE} in \eqref{eq:lie_sde} truncated after $N$ terms.  Let $\mbs{\Omega} = \mbs{\Omega}_1 + \mbs{\Omega}_2 + \cdots + \mbs{\Omega}_N$ be the discrete-time `Magnus matrix' obtained by applying the Magnus expansion to the linearized \ac{LTV} \ac{SDE} in~\eqref{eq:pert_sde} truncated after the same $N$ terms.  For simplicity, we let $t \in [0,T]$, $\mbs{\varpi}(t) = ( 1 - t/T)\mbs{\varpi}_1 + (t/T) \mbs{\varpi}_2$, and $\mbf{A}(t) = ( 1 - t/T)\mbf{A}_1 + (t/T) \mbf{A}_2$ such that
\begin{equation}
\mbf{A}(t) = \bbm \mbs{\varpi}(t)^\Wdg & \mbs{1} \\ \mbf{0} & \mbf{0} \ebm.
\end{equation}
Then we have that
\begin{equation}
    \mbs{\Omega} = \bbm \mbs{\psi}^\Wdg & \mbf{M} \\ \mbf{0} & \mbf{0} \ebm,
\end{equation}
where $\mbf{M} = \frac{\partial \mbs{\psi}}{\partial \mbs{\varpi}_2} + \frac{\partial \mbs{\psi}}{\partial \mbs{\varpi}_{1}}$ is the aggregate Jacobian of the Magnus vector with respect to the body-centric velocities.
\end{theorem}

\begin{proof}
We can verify the result term by term in the Magnus expansion.
For $n=1$ we have that $\mbs{\psi}_1 = \frac{T}{2} \left( \mbs{\varpi}_1 + \mbs{\varpi}_2 \right)$ and $\mbs{\Omega}_1 = \frac{T}{2} \bbm \left( \mbs{\varpi}_1 + \mbs{\varpi}_2\right)^\Wdg & 2 \cdot \mbs{1} \\ \mbf{0} & \mbf{0} \ebm$, so the result holds by inspection.  For $n=2$, we have that $\mbs{\psi}_2 = \frac{T^2}{12} \mbs{\varpi}_2^\Wdg \mbs{\varpi}_1$ and 
\begin{equation}
\mbs{\Omega}_2 = \frac{T^2}{12} \bbm [\mbs{\varpi}_2^\Wdg, \mbs{\varpi}_1^\Wdg] & \left(\mbs{\varpi}_2 - \mbs{\varpi}_1\right)^\Wdg \\ \mbf{0} & \mbf{0} \ebm .
\end{equation}  
Since
\begin{equation}
\mbs{\psi}_2 = \frac{T^2}{2}\mbs{\varpi}_2^\Wdg \mbs{\varpi}_1 , \quad \mbf{M}_2 = \frac{\partial \mbs{\psi}_2}{\partial \mbs{\varpi}_2} + \frac{\partial \mbs{\psi}_2}{\partial \mbs{\varpi}_{1}} = \frac{T^2}{12} \left( \mbs{\varpi}_2 - \mbs{\varpi}_1\right)^\Wdg,
\end{equation}
we again have the equivalence.  From here, we note that all the terms in the Magnus expansion involve right-nested commutators with the innermost term being $[\mbf{A}_2, \mbf{A}_1]$ \citep{arnal25}.  We can then proceed by induction.  We have shown the base case for $n=2$.  We then assume there is some commutator term at order $n$ for which the equivalence holds so that
\begin{equation}
    \mbs{\Omega}_n = \bbm \mbs{\psi}_n^\Wdg & \mbf{M}_n \\ \mbf{0} & \mbf{0} \ebm,
\end{equation}
where $\mbf{M}_n = \frac{\partial \mbs{\psi}_n}{\partial \mbs{\varpi}_2} + \frac{\partial \mbs{\psi}_n}{\partial \mbs{\varpi}_{1}}$.  Each term at order $n+1$ is formed by taking the commutator of the order $n$ term with a linear combination of $\mbs{\varpi}_1$ and $\mbs{\varpi}_2$, or $\mbf{A}_1$ and $\mbf{A}_2$, so that
\beqn{}
\mbs{\psi}_{n+1} & = & \left( w_1 \mbs{\varpi}_1 + w_2 \mbs{\varpi}_2 \right)^\Wdg \mbs{\psi}_n, \\
\mbs{\Omega}_{n+1} & = & \left[ w_1 \mbf{A}_1 + w_2 \mbf{A}_2, \; \mbs{\Omega}_n \right],
\eeqn
where $w_1$ and $w_2$ are scalar weights.  We have from the first equation by the product rule of differentiation that
\begin{equation}
    \mbf{M}_{n+1} = \frac{\partial \mbs{\psi}_{n+1}}{\partial \mbs{\varpi}_2} + \frac{\partial \mbs{\psi}_{n+1}}{\partial \mbs{\varpi}_{1}} = \left( w_1 \mbs{\varpi}_1 + w_2 \mbs{\varpi}_2 \right)^\Wdg \mbf{M}_n - \left( w_1 + w_2 \right) \mbs{\psi}_n^\Wdg.
\end{equation}
From the second equation, we have that 
\begin{multline}
    \mbs{\Omega}_{n+1} = \left[ \bbm \left( w_1 \mbs{\varpi}_1 + w_2 \mbs{\varpi}_2 \right)^\Wdg & (w_1 + w_2) \mbf{1} \\ \mbf{0} & \mbf{0} \ebm, \bbm \mbs{\psi}_n^\Wdg & \mbf{M}_n \\ \mbf{0} & \mbf{0} \ebm \right] \\ = \bbm \left[ \left( w_1 \mbs{\varpi}_1 + w_2 \mbs{\varpi}_2 \right)^\Wdg, \; \mbs{\psi}_n^\Wdg \right] & \left( w_1 \mbs{\varpi}_1 + w_2 \mbs{\varpi}_2 \right)^\Wdg \mbf{M}_n - \left( w_1 + w_2 \right) \mbs{\psi}_n^\Wdg \\ \mbf{0} & \mbf{0} \ebm = \bbm \mbs{\psi}_{n+1}^\Wdg & \mbf{M}_{n+1} \\ \mbf{0} & \mbf{0} \ebm,   
\end{multline}
which matches the desired form.  Thus, by induction, the result holds for all $N$.
\end{proof}

We have shown that the `Magnus matrix' obtained from the \ac{LTV} \ac{SDE} in the perturbations has the desired form.  Taking the matrix exponential then leads to the discrete-time state transition matrix:
\begin{equation}
    \mbs{\Phi}(T,0) = \exp\left( \mbs{\Omega} \right) = \bbm \exp\left( \mbs{\psi}^\Wdg \right) & \Jbig\left( \mbs{\psi} \right) \mbf{M} \\ \mbf{0} & \mbf{1} \ebm,
\end{equation}
where we have used the property of the matrix exponential of block upper-triangular matrices \citep{barfoot_ser24}.  This appears in the main body as \eqref{eq:stm}.

\end{document}